\documentclass{article}

\usepackage{microtype}
\usepackage{graphicx}
\usepackage{booktabs} 

\usepackage{subcaption}
\usepackage{amsmath}
\usepackage{amsthm}
\usepackage{amsfonts}
\newtheorem{theorem}{Theorem}
\newtheorem{proposition}{Proposition}
\usepackage{algorithm2e}
\usepackage{multirow}

\usepackage{xcolor}
 
\DeclareMathOperator*{\argmin}{arg\,min}

\usepackage{stmaryrd}

\usepackage{hyperref}

\usepackage[accepted]{icml2020}

\icmltitlerunning{Sub-Goal Trees -- a Framework for Goal-Based Reinforcement Learning}

\begin{document}

\twocolumn[
\icmltitle{Sub-Goal Trees -- a Framework for Goal-Based Reinforcement Learning}



\icmlsetsymbol{equal}{*}

\begin{icmlauthorlist}
\icmlauthor{Tom Jurgenson}{tech}
\icmlauthor{Or Avner}{tech}
\icmlauthor{Edward Groshev}{osaro}
\icmlauthor{Aviv Tamar}{tech}
\end{icmlauthorlist}

\icmlaffiliation{tech}{EE Department, Technion}
\icmlaffiliation{osaro}{Osaro Inc.}

\icmlcorrespondingauthor{Tom Jurgenson}{tomj@campus.technion.ac.il}
\icmlcorrespondingauthor{Aviv Tamar}{avivt@technion.ac.il}

\icmlkeywords{Machine Learning, ICML}

\vskip 0.3in
]



\printAffiliationsAndNotice{} 

\begin{abstract}
Many AI problems, in robotics and other domains, are goal-based, essentially seeking trajectories leading to various goal states. Reinforcement learning (RL), building on Bellman's optimality equation, naturally optimizes for a single goal, yet can be made multi-goal by augmenting the state with the goal. Instead, we propose a new RL framework, derived from a dynamic programming equation for the \emph{all pairs shortest path} (APSP) problem, which naturally solves multi-goal queries. We show that this approach has computational benefits for both standard and approximate dynamic programming.
Interestingly, our formulation prescribes a novel protocol for computing a trajectory: instead of predicting the next state given its predecessor, as in standard RL, a goal-conditioned trajectory is constructed by first predicting an intermediate state between start and goal, partitioning the trajectory into two. Then, recursively, predicting intermediate points on each sub-segment, until a complete trajectory is obtained. We call this trajectory structure a \emph{sub-goal tree}. Building on it, we additionally extend the policy gradient methodology to recursively predict sub-goals, resulting in novel goal-based algorithms. Finally, we apply our method to neural motion planning, where we demonstrate significant improvements compared to standard RL on navigating a 7-DoF robot arm between obstacles.
\end{abstract}

\section{Introduction}
Many AI problems can be characterized as learning or optimizing goal-based trajectories of a dynamical system, for example, robot skill learning and motion planning~\cite{mulling2013learning,gu2017deep,lavalle2006planning,qureshi2018motion}. The  reinforcement learning (RL) formulation, a popular framework for trajectory optimization based on Bellman's dynamic programming (DP) equation~\citep{bertsekas1996neuro}, naturally addresses the case of a single goal, as specified by a single reward function. RL formulations for multiple goals have been proposed~\citep{schaul2015universal,andrychowicz2017hindsight}, typically by augmenting the state space to include the goal as part of the state, but without changing the underlying DP structure.

On the other hand, in deterministic shortest path problems, multi-goal trajectory optimization is most naturally represented by the all-pairs shortest path (APSP) problem~\citep{russel2010AI}. In this formulation, augmenting the state and using Bellman's equation is known to be sub-optimal\footnote{This is equivalent to running the Bellman-Ford
single-goal algorithm for each possible goal state~\citep{russel2010AI}.}, and dedicated APSP algorithms such as Floyd-Warshall~\citep{floyd1962algorithm} build on different DP principles. Motivated by this, we propose a goal-based RL framework that builds on an efficient APSP solution, and is also applicable to large or continuous state spaces using function approximation.
\begin{figure*}[h]
\centering
\includegraphics[width=0.8\textwidth]{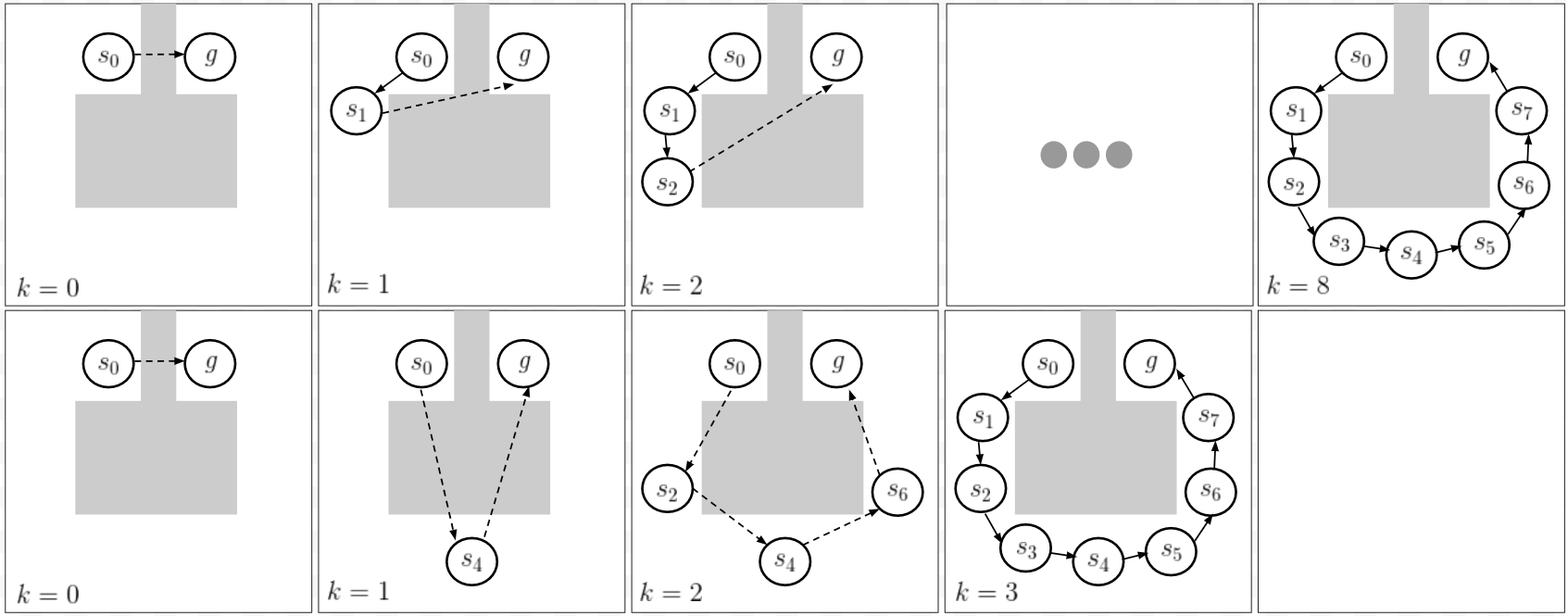}
\caption{Trajectory prediction methods. Upper row: a conventional \textit{Sequential } representation. Lower row: \textit{Sub-Goal Tree} representation. 
Solid arrows indicate predicted segments, while dashed arrows indicate segments that still require to be predicted.
By concurrently predicting sub-goals, a \textit{Sub-Goal Tree} only requires $3$ sequential computations, while the \textit{sequential} requires $8$.
}
\label{fig:one_step_vs_trajsplit}
\vspace{-5mm}
\end{figure*}

Our key idea is that a goal-based trajectory can be constructed in a divide-and-conquer fashion. First, predict a sub-goal between start and goal, partitioning the trajectory into two. Then, recursively, predict intermediate sub-goals on each sub-segment, until a complete trajectory is obtained. We call this trajectory structure a \emph{sub-goal tree}
(Figure~\ref{fig:one_step_vs_trajsplit}),
and we develop a DP equation for APSP that builds on it and is compatible with function approximation.
We further bound the error of following the sub-goal tree trajectory in the presence of approximation errors, and show favorable results compared to the conventional RL method, intuitively, due to the sub-goal tree's lower sensitivity to drift.

The sub-goal tree can also be seen as a general parametric structure for a trajectory, where a parametric model, e.g., a neural network, is used to predict each sub-goal given its predecessors. Based on this view, we develop a policy gradient framework for sub-goal trees, and show that conventional policy-gradient techniques such as control variates and trust regions~\citep{greensmith2004variance, schulman2015trust} naturally apply here as well. 

Finally, we present an application of our approach to neural motion planning -- learning to navigate a robot between obstacles~\cite{qureshi2018motion,jurgenson2019harnessing}. Using our policy gradient approach, we demonstrate navigating a 7-DoF continuous robot arm safely between obstacles, obtaining marked improvement in performance compared to conventional RL approaches.

\section{Related work}
In RL, the idea of sub-goals has mainly been investigated under the options framework~\cite{sutton1999between}. In this setting, the goal is typically fixed (i.e., given by the reward in the MDP), and useful options are discovered using some heuristic such as bottleneck states~\cite{mcgovern2001automatic, menache2002q} or changes in the value function~\cite{konidaris2012robot}. 
While hierarchical RL using options seems intuitive,
theoretical results for their advantage are notoriously difficult to establish, and current results require non-trivial assumptions on the option structure~\cite{mann2014scaling,fruit2017regret}. Interestingly, by investigating the APSP setting, we obtain general and strong results for the advantage of sub-goals. 

Universal value functions~\cite{schaul2015universal,andrychowicz2017hindsight} learn a goal-conditioned value function using the Bellman equation. In contrast, in this work we propose a principled motivation for sub-goals based on the APSP problem, and develop a new RL formulation based on this principle. A connection between RL and the APSP has been suggested by~\citet{kaelbling1993learning,dhiman2018floyd}, based on the Floyd-Warshall algorithm. However, as discussed in Section~\ref{sec:approx-sgtdp}, these approaches become unstable once function approximation is introduced. Goal-conditioned policies are also related to universal plans in the classical planning literature~\cite{schoppers1989representation,kaelbling1988goals,ginsberg1989universal}, and in this sense SGTs can be used to approximate universal plans using learning.

Sub-goal trees can be seen as a form of trajectory representation (as we discuss in Section~\ref{sec:sgt-pg}). Such has been investigated for learning robotic skills~\cite{mulling2013learning, sung2018robobarista} and navigation~\cite{qureshi2018motion}.
A popular approach, 
dynamical movement primitives
~\cite{ijspeert2013dynamical}, represents a continuous trajectory as a dynamical system with an attractor at the goal, and was successfully used for RL~\cite{kober2013reinforcement,mulling2013learning,peters2008reinforcement}. 
The temporal segment approach of~\citet{mishra2017prediction}, on the other hand, predicts segments of a trajectory sequentially.
In the context of video prediction, \citet{jayaraman2019time} proposed to predict salient frames in a goal-conditioned setting by a supervised learning loss that focuses on the `best' frames. This was used to predict a list of sub-goals for a tracking controller. In contrast, we investigate the RL problem,  and predict sub-goals \textit{recursively}.

Approximate APSP solutions have been traditionally explored in the non-learning setting. Approaches such as in~\citet{thorup2005approximate, chan2010more, williams2014faster} compute $k$-optimal paths ($k>1$) in running time that is sub-cubic in the number of vertices in the graph. These methods require full knowledge of the graph and process each vertex at least once. In contrast, our learning-based approach can generalize to unseen states (vertices) by learning from similar states, potentially scaling to very large or even continuous state spaces, as we demonstrate in our experiments.

\section{Problem Formulation and Notation}
We are interested in optimizing goal-conditioned tasks in dynamical systems. We restrict ourselves to deterministic, stationary, and finite time systems, and consider both discrete and continuous state formulations.\footnote{The deterministic setting is fundamental to our approach, and therefore we do not build on the popular Markov decision process framework. We leave the investigation of similar ideas to stochastic and time varying systems for future work. We note that deterministic systems are popular in RL applications, such as~\citet{bellemare2013arcade, duan2016benchmarking}.}

In the discrete state setting, we study the all-pairs shortest path (APSP) problem on a graph. Consider a directed and weighted graph with $N$ nodes $s_1,\dots, s_N$, and denote by $c(s,s')\geq 0$ the weight of edge $s\to s'$.\footnote{Technically, and similarly to standard APSP algorithms~\cite{russel2010AI}, we only require that there are no negative cycles in the graph. To simplify our presentation, however, we restrict $c$ to be non-negative. } By definition, we set $c(s,s) = 0$, such that the shortest path from a node to itself is zero.
To simplify notation, we can replace unconnected edges by edges with weight $\infty$. Thus, without loss of generality, throughout this work we assume that the graph is complete. The APSP problem seeks the shortest paths (i.e., a path with minimum sum of costs) \textbf{from any start node $s$ to any goal node $g$} in the graph:
\vspace{-5px}
\begin{equation}\label{eq:discrete_sp}
        \min_{T, s_0=s,s_1,\dots,s_{T-1},s_T=g} \sum_{t=0}^{T-1} c(s_t, s_{t+1}).
\end{equation}
\vspace{-13px}

In the continuous state case, we consider a deterministic controlled dynamical system in discrete time, defined over a state space $S$ and an action space $U$: 
where $s_t, s_{t+1}\in S$ are the states at time $t$ and $t+1$, respectively, $u_t \in U$ is the control at time $t$, and $f$ is a stationary state transition function. 
Given an initial state $s_0$ and a goal $g$, an optimal trajectory reaches $g$ while minimizing the sum of non-negative costs $\bar{c}(s,u)\geq 0$:
\vspace{-5px}
\begin{equation}\label{eq:opt_control}  
        \min_{T, u_0,\dots,u_{T}} \sum_{t=0}^T \bar{c}(s_t, u_t), 
        \textrm{ s.t. }
        {s_{t + 1}}\! =\! {f}({s_t},{u_t}), 
        {s_{T}}\! = \!g.
\end{equation}
\vspace{-13px}

To unify our treatment of problems \eqref{eq:discrete_sp} and \eqref{eq:opt_control}, in the remainder of this paper we abstract away the actions as follows. Let $c(s,s') = \min_u \bar{c}(s, u) \textrm{ s.t. } {s'} = {f}({s},{u})$, and $c(s,s')=\infty$ if no transition from $s$ to $s'$ is possible. Then, for any start and goal states $s,g$, Eq.~\eqref{eq:opt_control} is equivalent to:\footnote{In practice, and as we report in our experiments, such an abstraction can be easily implemented using an inverse model.}
\vspace{-5px}
\begin{equation}\label{eq:cont_sp}
        \min_{T, s_0=s,s_1,\dots,s_{T-1},s_T=g} \sum_{t=0}^{T-1} c(s_t, s_{t+1}).
\end{equation}
Note the similarity of Problem \eqref{eq:cont_sp} to the APSP problem \eqref{eq:discrete_sp}, where feasible transitions of the dynamical system are represented by edges between states in the graph. In the rest of this paper, we study solutions for \eqref{eq:discrete_sp} and \eqref{eq:cont_sp}, for all start and goal pairs $s,g$. We are particularly interested in large problems, where exact solutions are intractable, and approximations are required.

\textbf{Notation:} Let $\tau\!=\!s_0,\dots,s_T$ denote a state trajectory. 
Also, denote $c_{i:j}\!=\!\sum_{t=i}^{j-1}{\!c(s_t, s_{t+1})}$, the cost of several subsequent states and $c_{\tau}=c_{0:T}$ as the total trajectory cost. 
Let $\tau(s_0,s_m)\!\!=\!\!s_0,\dots,s_m$ denote the segment of $\tau$ starting at $s_0$ and ending at $s_m$. 
We denote the trajectory concatenation as $\tau \!\!=\!\! [\tau(s_0,s_m), \tau(s_m,s_T)]$, where it is implied that $s_m$ appears only once. 
For a stochastic trajectory distribution conditioned on start and goal, we denote  $\Pr_{\pi}[s_i,\dots, s_j | s, g]$ as shorthand for $\Pr_{\pi}[s_i,\dots, s_j | s_i\!=\!s,\!s_j\!=\!g]$.

\section{Approximate Dynamic Programming}

In this section we study a dynamic programming principle for the APSP problem \eqref{eq:discrete_sp}. We will later extend our solution to the continuous case \eqref{eq:cont_sp} using function approximation.

One way to solve \eqref{eq:discrete_sp} is to solve a single-goal problem for every possible goal state. The RL equivalent of this approach is UVFA, suggested by \citet{schaul2015universal}, where the state space is augmented with the goal state, and a goal-based value function is introduced. 
The UVFA is optimized using standard (single-goal) RL, by learning over different goals. 
Alternatively, our approach builds on an efficient APSP dynamic programming algorithm, as we propose next.

\subsection{Sub-Goal Tree Dynamic Programming}\label{sec:SGT_DP}


By our definition of non-negative costs, the shortest path between any two states is of length at most $N$. The main idea in our approach is that a trajectory between $s$ and $g$ can be constructed in a divide-and-conquer fashion: first predict a sub-goal between start and goal that optimally partitions the trajectory to two segments of length $N/2$ or less: $s,\dots,s_{N/2}$ and $s_{N/2},\dots,g$. Then, recursively, predict intermediate sub-goals on each sub-segment, until a complete trajectory is obtained. We term this composition a \textit{sub-goal tree} (SGT), and we formalize it as follows.

Let $V_k(s, s')$ denote the shortest path from $s$ to $s'$ in $2^k$ steps or less. Note that by our convention about unconnected edges above, if there is no such trajectory in the graph then $V_k(s, s')=\infty$. We observe that $V_k$ obeys the following dynamic programming relation, which we term \emph{sub-goal tree dynamic programming} (SGTDP).
\begin{theorem}\label{thm:SGTDP}
Consider a weighted graph with $N$ nodes and no negative cycles. Let $V_k(s, s')$ denote the cost of the shortest path from $s$ to $s'$ in $2^k$ steps or less, and let $V^*(s,s')$ denote the cost of the shortest path from $s$ to $s'$. Then, $V_k$ can be computed according to the following equations: 
\begin{equation}\label{eq:DP_trajsplit}
    \begin{split}
        V_0(s, s') &= c(s, s'), \quad \forall s,s': s\neq s'; \\
        V_k(s, s)\text{ } &= 0, \quad \forall s; \\
        V_k(s, s') &= \!\min_{s_m} \!\left\{ V_{k-1}(s, s_m) \!+\! V_{k-1}(s_m, s')\right\}\!, \forall s,\!s' \!\!:\! s\!\neq \!s' \!.
    \end{split}\raisetag{4em}
\end{equation}
Furthermore, for $k \geq \log_2(N)$ we have that $V_k(s, s') = V^*(s, s')$ for all $s,s'$. 
\end{theorem}

The proof of Theorem \ref{thm:SGTDP}, along with all other proofs in this paper, 
is
reported in the supplementary material.

Given $V_0,\dots,V_{\log_2(N)}$, Theorem \ref{thm:SGTDP} prescribes the following recipe for computing a shortest path for every $s,g$, which we term the \textbf{greedy SGT trajectory}:\footnote{Note that the cost of the greedy SGT trajectory is optimal, however, its length is by definition $N$. Thus, a shorter-length trajectory with the same cost may exist. Simple modifications of SGT to find the shortest-length cost optimal trajectory are possible, but we do not consider them here. }
\begin{small}
\begin{align} \label{eq:greedy_traj}
        s_0 &= s,\text{ }s_N = g \\ \nonumber
        s_{N/2} &\in \argmin_{s_m} \!\left\{\! V_{\log_2 (N)-1}(\!s_0, s_m\!) \!+\! V_{\log_2 (N)-1}(s_m, s_N)\!\right\}\!, \\ \nonumber
        s_{N/4} &\in \argmin_{s_m} \!\left\{\! V_{\log_2 (N)-2}(\!s_0, s_m\!) \!+\! V_{\log_2 (N)-2}(s_m, s_{N/2})\!\right\}\!,
\end{align}
\end{small}
\begin{small}
\begin{align}
        \nonumber
        s_{3N/4} &\in \argmin_{s_m} \!\left\{\! V_{\log_2 (N)-2}(\!s_{N/2}, s_m\!) \!+\! V_{\log_2 (N)-2}(s_m, s_{N})\!\right\}\!, \\
        \nonumber
        \dots& 
\end{align}\raisetag{6em}
\end{small}
The SGT can be seen as a general divide-and-conquer representation of a trajectory, by recursively predicting sub-goals. In contrast, the construction of a greedy trajectory in conventional RL (using the Bellman equation) proceeds \emph{sequentially}, by predicting state (or action) after state. This is illustrated in Figure~\ref{fig:one_step_vs_trajsplit}. We next investigate the advantages of the SGT structure compared to the sequential approach. 

In conventional RL, it is well known that in the presence of errors, following a greedy policy may suffer from \emph{drift} -- accumulating errors that may hinder performance~\citep{ross2011reduction}.
In the goal-based setting, intuitively, the error magnitude scales with the distance from the goal. Our main observation is that the SGT is less sensitive to drift, as the sub-goals break the trajectory into smaller segments with exponentially decreasing errors. We next formalize this idea, by analyzing an approximate DP formulation.

\subsection{SGT Approximate DP}\label{sec:approx-sgtdp}

Similar to standard dynamic programming algorithms, the SGTDP algorithm iterates over all states in the graph, which becomes infeasible for large, or continuous state spaces. For such cases, inspired by the approximate dynamic programming (ADP) literature~\cite{bertsekas2005dynamic}, we investigate ADP methods based on SGTDP. 

Let $T$ denote the SGTDP operator, $(TV)(s,s') = \min_{s_m} \left\{ V(s, s_m) + V(s_m, s')\right\}$. From Theorem \ref{thm:SGTDP}, we have that $V^* = T^{\log_2(N)}V_0$. Similarly to standard ADP analysis~\cite{bertsekas2005dynamic}, we assume an $\epsilon$-approximate SGTDP operator that generates a sequence of approximate value functions $\hat{V}_0,\dots, \hat{V}_{\log_2(N)}$ that satisfy:
\begin{equation}\label{eq:error}
    \| \hat{V}_{k+1} - T \hat{V}_k \|_\infty \leq \epsilon, \quad \quad \| \hat{V}_{0} - V_0 \|_\infty \leq \epsilon,
\end{equation}
where $\|x\|_\infty = \max_{s,s'} |x(s,s')|$. The next result provides an error propagation bound for SGTDP.\footnote{For conventional RL, $\|\cdot\|_p$ errors bounds were developed by \citet{munos2008finite}, and are more suitable for the learning setting. We leave such investigation to future work, and emphasize that our simpler $\|\cdot\|_\infty$ bounds still show the fundamental soundness of our approach. } 

\begin{proposition}\label{prop:err_prop}
For the sequence $\hat{V}_0,\dots, \hat{V}_{\log_2(N)}$ satisfying \eqref{eq:error}, we have that $\| \hat{V}_{\log_2(N)} - V^* \|_\infty \leq \epsilon(2N-1)$.
\end{proposition}

A similar $O(\epsilon N)$ error bound is known for ADP based on the approximate Bellman operator~\citep[][Pg. 332]{bertsekas1996neuro}.
Thus, Proposition \ref{prop:err_prop} shows that given the same value function approximation method, we can expect a similar error in the value function between the SGT and sequential approaches. However, the main importance of the value function is in deriving a policy, in this case, a trajectory from start to goal. As we show next, the shortest path derived from the approximate SGTDP value function can be significantly more accurate than a path derived using the Bellman value function.

We first discuss 
how to compute
a greedy trajectory. Given a sequence of $\epsilon$-approximate SGTDP value functions $\hat{V}_0,\dots, \hat{V}_{\log_2(N)}$ as described above, one can compute the greedy SGT trajectory by plugging the approximate value functions in \eqref{eq:greedy_traj} instead of their respective accurate value functions. We term this the \textbf{approximate greedy SGT trajectory}, and provide an error bound for it.

\begin{proposition}\label{prop:STDP_error}
For a start and goal pair $s,g$, and sequence of value functions $\hat{V}_0,\dots, \hat{V}_{\log_2(N)}$ generated by an $\epsilon$-approximate SGTDP operator, let $s_0,\dots,s_N$ denote the approximate greedy SGT trajectory as described above.
We have that $\sum_{i=0}^{N-1} c(s_i,s_{i+1}) \leq V^*(s,g) + 4N \log_2{(N)} \epsilon $. 
\vspace{-8px}
\end{proposition}

Thus, the error of the greedy SGT trajectory is $O(N \log_2{(N)})$. In contrast, for the greedy trajectory according to the standard finite-horizon Bellman equation, a tight $O(N^2)$ bound holds~\citep[][we also provide an independent proof in the supplementary material]{ross2011reduction}. Thus, the SGT approach provides a strong improvement in handling errors compared to the sequential method. In addition, the SGT approach requires us to compute and store only $\log_2(N)$ different value functions. In comparison, a standard finite horizon approach would require storing $N$ value functions, which can be limiting for large $N$.
Finally, during trajectory prediction, sub-goal predictions for different branches of the tree are independent and can be computed concurrently, allowing an $O(\log_2(N))$ prediction time, whereas sequentially predicting sub-goals is $O(N)$. We thus conclude that the SGT provides significant benefits in both approximation accuracy, prediction time, and space.

\paragraph{Why not use the Floyd-Warshall Algorithm?}
At this point, the reader may question why we do not build on the Floyd-Warshall (FW) algorithm for the APSP problem. The FW method maintains a value function $V_{FW}(s,s')$ of the shortest path from $s$ to $s'$, and updates the value using the relaxation $V_{FW}(s,s') := \min \left\{V_{FW}(s,s'), V_{FW}(s,s_m)+V_{FW}(s_m,s')\right\}$. If the updates are performed over all $s_m, s,$ and $s'$ (in that sequence), $V_{FW}$ will converge to the shortest path, requiring $O(N^3)$ computations, compared to $O(N^3\log N)$ for SGTDP~\citep{russel2010AI}. One can also perform relaxations in an arbitrary order, as was suggested by \citet{kaelbling1993learning}, and more recently by~\citet{dhiman2018floyd}, to result in an RL style algorithm. However, as was already observed by~\citet{kaelbling1993learning}, the FW relaxation requires that the values always over-estimate the optimal costs, and any under-estimation error, due to noise or function approximation, gets propagated through the algorithm without any way of recovery, leading to instability. Our SGT approach avoids this problem by updating $V_k$ based on $V_{k-1}$, resembling finite horizon dynamic programming, and, as we proved, maintains stability in presence of errors. Furthermore, both \citet{kaelbling1993learning,dhiman2018floyd} showed results only for table-lookup value functions. In our experiments, we have found that replacing the SGTDP update with a FW relaxation (described in the supplementary) leads to instability when used with function approximation.

\section{Sub-Goal Tree RL Algorithms}
Motivated by the theoretical results of the previous section, we present algorithms for learning SGT policies. We start with a value-based batch-RL algorithm with function approximation in Section~\ref{sec:batch-rl-sgt}. Then, in Section~\ref{sec:sgt-pg}, we present a policy gradient approach for SGTs.

\subsection{Batch RL with Sub-Goal Trees}\label{sec:batch-rl-sgt}
We now describe a batch RL algorithm with function approximation based on the SGTDP algorithm above. Our approach is inspired by the fitted-Q iteration (FQI) algorithm for finite horizon Markov decision processes (\citealt{tsitsiklis2001regression}; see also~\citealt{ernst2005tree,riedmiller2005neural} for the discounted case). 
Similar to FQI, we are given a data set of $M$ random state transitions and their costs $\{(s_i,s'_i,c_i )\}_{i=1}^M$, and we want to estimate $V_k(s,s')$ for arbitrary pairs $(s,s')$.
Assume that we have some estimate $\hat{V}_k(s,s')$ of the value function of depth $k$ in SGTDP. Then, for any pair of start and goal states $s,g$, we can estimate $\hat{V}_{k+1}(s,g)$ as 
\vspace{-0.5em}
\begin{equation}\label{eq:STDP_min}
    \hat{V}_{k+1}(s,g) = \min_{s_m} \left\{ \hat{V}_{k}(s, s_m) + \hat{V}_{k}(s_m, g)\right\}.
\end{equation} 
Thus, if our data consisted of start and goal pairs, we could use \eqref{eq:STDP_min} to generate \emph{regression targets} for the next value function, and use any regression algorithm to fit $\hat{V}_{k+1}(s,s')$. This is the essence of the Fitted SGTDP algorithm (Algorithm \ref{alg:approximate_STDP__}). Since our data does not contain explicit goal states, we simply define goal states to be randomly selected states from within the data (lines 6,7 in Alg.~\ref{alg:approximate_STDP__}).

The first iteration $k=0$ in Fitted SGTDP, however, requires special attention. We need to fit the cost function for connected states in the graph, yet make sure that states which are not reachable in a single transition have a high cost. To this end, we fit the observed costs $c$ to the observed transitions $s,s'$ in the data (lines 2,5 in Alg.~\ref{alg:approximate_STDP__}), and a high cost $C_{max}$ to transitions from the observed states to randomly selected states (lines 3,5 in Alg.~\ref{alg:approximate_STDP__}). We also fit a cost of zero to self transitions (lines 4,5 in Alg.~\ref{alg:approximate_STDP__}).

Our algorithm also requires a method to approximately solve the minimization problem in \eqref{eq:STDP_min}. In our experiments, we discretized the state space and performed a simple grid search. Other methods could be used in general. For example, if $V_k$ is represented as a neural network, then one can use gradient descent. Naturally, the quality of Fitted SGTDP will depend on the quality of solving this minimization problem.

\begin{algorithm}[htp]\caption{Fitted SGTDP}
  \SetAlgoLined\DontPrintSemicolon
  \SetKwProg{myalg}{Algorithm}{}{}
  \label{alg:approximate_STDP__}
  \myalg{ }{
  \setcounter{AlgoLine}{0}
  \nl Input: dataset $D = \left\{ s, c, s'\right\}$, max path cost $C_{max}$ \\
  \nl Create real transition data:\\
  $D_{trans} = \left\{ s, s'\right\}$ with targets $T_{trans} = \left\{ c\right\}$ from $D$\\
  \nl Create fake transition data:\\ $D_{random} = \left\{ s, s_{rand}\right\}$ with targets $T_{random} = \left\{ C_{max}\right\}$ with $s,s_{rand}$ random states from $D$\\
  \nl Create self transition data:\\ $D_{self} = \left\{ s, s\right\}$ with targets $T_{self} =\left\{ 0\right\}$ with $s$ taken from $D$\\
  \nl Fit $\hat{V}_0(s,s')$ to data in $D_{trans}, D_{random}, D_{self}$ and targets $T_{trans}, T_{random}, T_{self}$\\
  \For{$k: 1...K$}{
  \nl Create goal data $D_{goal} \!=\! \left\{ s, g\right\}$ and targets 
$T_{goal} \!=\! \{ \min_{s_m} \!\!\{ \hat{V}_{k-1}(s, s_m) \!+\! \hat{V}_{k-1}(s_m, g)\}\}$ with $s,g$ randomly chosen from states in $D$\\
  \nl Fit $\hat{V}_k(s,s')$ to data in $D_{goal}$ and targets in $T_{goal}$\\
  }  }
\end{algorithm}
\vspace{-10px}

\subsection{Sub-Goal Trees Policy-Gradient}\label{sec:sgt-pg}

The minimization over sub-goal states in Fitted SGTDP can be difficult for high-dimensional state spaces. 
A similar problem arises in standard RL with continuous actions~\cite{bertsekas2005dynamic,kalashnikov2018qt}, and has motivated the study of policy search methods, which directly optimize the policy~\citep{deisenroth2013survey}. 
Following a similar motivation, we propose a policy search approach based on SGTs.
Inspired by policy gradient (PG) algorithms in conventional RL~\cite{sutton2000policy,deisenroth2013survey}, we propose a parametrized stochastic policy that approximates the optimal sub-goal prediction, and we develop a corresponding PG theorem for training the policy.

\textbf{Stochastic SGT policies:} 
A stochastic SGT policy $\pi(s' | s_1,s_2)$ is a stochastic mapping from two endpoint states $s_1, s_2\!\in\! S$ to a predicted sub-goal $s'\!\in\!S$.
Given a start state $s_0\!=\!s$ and goal $s_T\!=\!g$, the likelihood of $\tau\!=\!s_0, s_1, \dots, s_T$ under policy $\pi$ is defined recursively by: 
\begin{align}\label{eq:recursive-trajectory-likelihood}
&\Pr_{\pi}[\tau | s, g]= \Pr_{\pi}[s_0, \dots, s_T | s, g] \\
&= \Pr_{\pi}[s_0, \dots, s_{\frac{T}{2}} | s, s_m] 
    \Pr_{\pi}[s_{\frac{T}{2}}, \dots, s_T | s_m, g] \pi(s_m | s, g),\nonumber
\end{align}\raisetag{4em}
where the base of the recursion is $\Pr_{\pi}[s_t,s_{t+1} | s_t, s_{t+1}]=1$ for all $t\in[0,T-1]$.
This formulation assumes that sub-goal predictions within a segment depend only on states within the segment, and not on states before or after it. This is analogous to the Markov property in conventional RL, where the next state prediction depends only on the previous state, but adapted to a goal-conditioned setting.

Note that this recursive decomposition can be interpreted as a tree. 
Without loss of generality, we assume this tree has a depth of $D$ thus $T=2^D$ 
(repeating states to make the tree a full binary tree does not incur extra cost as $c(s,s)=0$).

We now define the PG objective. 
Let $\rho_0$ denote a distribution over start and goal pairs $s,g \in S$. our goal is to learn a stochastic policy $\pi_{\theta}(s' | s_1,s_2)$, characterized by a parameter vector $\theta$, that minimizes the expected trajectory costs:
\begin{align}\label{eq:cost-function-J}
    J(\theta) 
    = J^{\pi_{\theta}} 
    = &\mathbb{E}_{\tau\sim\rho(\pi_{\theta})}\left[c_{\tau} \right],
\end{align}
where the trajectory distribution $\rho(\pi_{\theta})$ is defined by first drawing $(s,g)\sim \rho_0$, and then recursively drawing intermediate states as defined by Eq.~\eqref{eq:recursive-trajectory-likelihood}.
Our next result is a PG theorem for SGTs. 

\begin{theorem}\label{thm:sgt-pg}
Let $\pi_{\theta}$ be a stochastic SGT policy,  $\rho(\pi_{\theta})$ be a trajectory distribution defined above, and $T=2^D$.
Then 
\begin{align}\label{eq:pg-thm-gradient-short}
    &\nabla_{\theta}J(\theta)=\mathbb{E}_{\rho(\pi_{\theta})}\left[ c_{\tau}\cdot \nabla_{\theta}\log{\Pr_{\rho(\pi_{\theta})}[\tau] }\right]  \\
    &=\mathbb{E}_{\rho(\pi_{\theta})}\left[
    \sum_{d=1}^{D}{ \sum_{i=1}^{2^{D-d}}{
    C_{\tau}^{i,d}\!\cdot \!
    \nabla_{\theta}\log{\pi_{\theta}\left(s_m^{i,d} \bigg\vert s^{i,d}, g^{i,d}\right) } } }
    \right]\!, \nonumber
\end{align}
where $s^{i,d}=s_{(i-1)\cdot 2^d}$, $s_m^{i,d}=s_{(2i-1)\cdot 2^{d-1}}$, $g^{i,d}=s_{i\cdot 2^d}$, and $C_{\tau}^{i,d}=c_{(i-1)\cdot 2^d: i\cdot 2^d}$ is the sum of costs from $s^{i,d}$ to $g^{i,d}$ of $\tau$.
Furthermore, let the baseline $b^{i,d}=b(s^{i,d}, g^{i,d})$ be any fixed function $b^{i,d}:S^2\rightarrow R$, then $C_{\tau}^{i,d}$ in \eqref{eq:pg-thm-gradient-short} can be replaced with $C_{\tau}^{i,d} - b^{i,d}$.
\end{theorem}

The main difference between Theorem \ref{thm:sgt-pg} and the standard PG theorem~\citep{sutton2000policy} is in the calculation of $\nabla_{\theta}\log{\Pr_{\rho(\pi_{\theta})}[\tau]}$, which in our case builds on the tree-based construction of the SGT trajectory.
The summations over $i$ and $d$ in Theorem~\ref{thm:sgt-pg}, are used to explicitly state individual sub-goal predictions: $d$ iterates over the depth of the tree, and $i$ iterates between sub-goals of the same depth.

Using Theorem~\ref{thm:sgt-pg}, we can sample a trajectory $\tau$ using $\pi_{\theta}$, obtain its cost $c_{\tau}$, and estimate of the gradient of $J(\theta)$ using $c_{\tau}$ and the gradients of individual decisions of $\pi_{\theta}$. 

In the policy gradients literature, variance reduction using control variates, and trust-region methods such as TRPO and PPO play a critical role~\citep{greensmith2004variance,schulman2015trust,schulman2017proximal}. 
The baseline reduction in Theorem~\ref{thm:sgt-pg} allows similar results to be derived for SGT, and we similarly find them to be important in practice. Due to space constraints, we report these in the
supplementary.

\subsection{The SGT-PG Algorithm}

\begin{algorithm}[htp]\caption{SGT-PG}
  \SetAlgoLined\DontPrintSemicolon
  \SetKwProg{myalg}{Algorithm}{}{}
  \SetKwFunction{collect}{collect}
  \SetKwFunction{computePG}{compute-PG}
  \SetKwFunction{optimizerStep}{optimizer.step}
  \SetKwProg{myproc}{Procedure}{}{}
  \SetKwFunction{predictSubgoals}{predict-subgoals}
  \SetKwFunction{eEval}{evaluate}
  \label{alg:sgt-pg}
  \myalg{ }{
  \setcounter{AlgoLine}{0}
  \nl Input: $D$ - depth, $N$ - episodes per cycle, $E$ - environment \\
  \nl init $\pi_1,\dots, \pi_D$ with parameters $\theta_1,\dots,\theta_D$ \\
  \nl \For{$d: 1...D$}{
  \nl \If{$d>1$}{
    \nl $\theta_d \leftarrow \theta_{d-1}$ \tcp{init $\pi_d$ from $\pi_{d-1}$}
  }
  \nl \While{convergence criterion for $\pi_d$ not met 
  }{
   \nl $\mathbb{D}= \{(\tau_i,c_{\tau_i}) \}_{i=1}^N \leftarrow$\collect{$d, N, E$} \\
   \nl $\nabla_{\theta_d}J(\theta_d)\leftarrow$\computePG{$\theta_d, \mathbb{D}$}\\
   \nl $\theta_d\leftarrow$\optimizerStep{$\theta_d,  \nabla_{\theta_d}J(\theta_d)$}
  }  }  }
  \setcounter{AlgoLine}{0}
  \myproc{\collect{$d$, $N$, $E$}}{
  \nl \For{$i: 1\dots N$}{
  \nl $(s^i,g^i)\sim \rho_0$ \\
  \nl $s_1^i,\dots,s_{2^d-1}^i\leftarrow$\predictSubgoals{$s^i,g^i,d$}\\
  \nl $\tau_i=[s^i, s_1^i,\dots,s_{2^d-1}^i, g^i]$ \\
  \nl $c_{\tau_i}=[c_{j:j+1} ]_{j=1}^{2^d-1}\leftarrow E.$\eEval{$\tau_i$}
  }
  \nl \Return $\{(\tau_i, c_{\tau_i) } \}_{i=1}^N$
  }
\end{algorithm}

Following the theoretical results in Section \ref{sec:SGT_DP}, where the greedy SGT trajectory used a
different value function to predict sub-goals at different depths of the tree, we can expect a stochastic SGT policy that depends on the depth in the tree ($d$ in Theorem \ref{thm:sgt-pg}) to perform better than a depth-independent policy. Theorem~\ref{thm:sgt-pg} holds true when $\pi$ depends on $d$, similarly to a time-dependent policy in standard PG literature. We denote such a depth dependent policy as $\pi_d$, for $d\in 1,\dots,D$.
At depth $d=0$, no sub-goal is predicted resulting in $(s,g)$ being directly connected. 
Next, $\pi_1$ predicts a sub-goal $s_m$, segmenting the trajectory to two segments of depth 0: $(s,s_m)$ and $(s_m,g)$.
The recursive construction continues, $\pi_2$ predicts a sub-goal and calls $\pi_1$ on the resulting segments, and so on until depth $D$.

An observation that we found important for improving training stability, is that the policies can be trained sequentially. Namely, we first train the $d$-depth policy, and only then start training $d+1$-depth policy, while freezing all policies of depth $\leq d$. 

Our algorithm implementation, SGT-PG, is detailed in Algorithm~\ref{alg:sgt-pg}.
SGT-PG predicts sub-goals, and interacts with an environment $E$ that evaluates the cost of segments $(s, s')$.
SGT-PG maintains $D$ depth-specific policies:
$\{\pi_i\}_{i=1}^D$
, each  parametrized by 
$\{\theta_i \}_{i=1}^D$
respectively (i.e., we do not share parameters for policies at different depths, though this is possible in principle).
The policies are trained in sequence, and every training cycle is comprised of on-policy data collection followed by policy update.
The \collect method collects on-policy data given $N$, the number of episodes to generate; $d$, the index of the policy being trained; and $E$, the environment.
We found that for reducing noise when training $\pi_d$, it is best to sample from $\pi_d$ only and take the mean predictions of 
$\{\pi_i\}_{i=1}^{d-1}$.
The \computePG method, based on Eq.~\eqref{eq:pg-thm-gradient-short}, uses the collected data $\mathbb{D}$, and estimates $\nabla_{\theta_d}J(\theta_d)$. 
We found that similar to sequential RL, adding a trust region using the PPO optimization objective~\citep{schulman2017proximal} provides stable updates (see Section~\ref{sec:nn-arch} for specific loss function).
Finally, \optimizerStep updates $\theta_d$ according to any SGD optimizer algorithm, which completes a single training cycle for $\pi_d$.
We proceed to train the next policy $\pi_{d+1}$ when a pre-specified convergence criterion is met, for instance, until the expected cost of predicted trajectories $\mathbb{E}\left[ c_{\tau} \right]$ stops decreasing.

\section{Experiments}
In this section we compare our \textit{SGT} approach with the conventional \textit{sequential} method (i.e. predicting the next state of the trajectory).
We consider APSP problems inspired
by
robotic motion planning, where the goal is to find a collision-free trajectory between a pair of start and goal states.
We illustrate the SGT value functions and trajectories on a simple 2D point robot domain, which we solve using Fitted SGTDP (Section \ref{sec:batch-rl-experiments}, code: \texttt{https://github.com/tomjur/SGT\_batch\_RL}\linebreak\texttt{.git}).
We then consider a more challenging domain with a simulated 7DoF robotic arm, and demonstrate the effectiveness of SGT-PG (Section \ref{sec:sgt-pg-experiments} code: \texttt{https://github.com/tomjur/SGT-PG.git}).


\begin{figure}
\centering
    \hfill
    \begin{subfigure}[b]{0.45\linewidth}
    \includegraphics[width=1.0\textwidth]{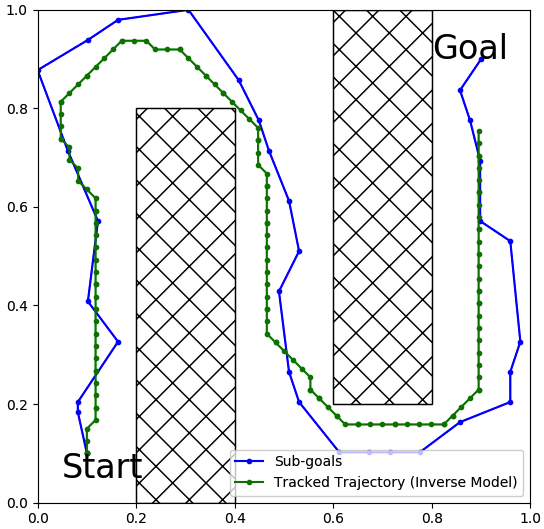}
    \caption{}
    \label{fig:rl}
    \end{subfigure}
    \hfill
    \begin{subfigure}[b]{0.45\linewidth}
    \includegraphics[width=\textwidth]{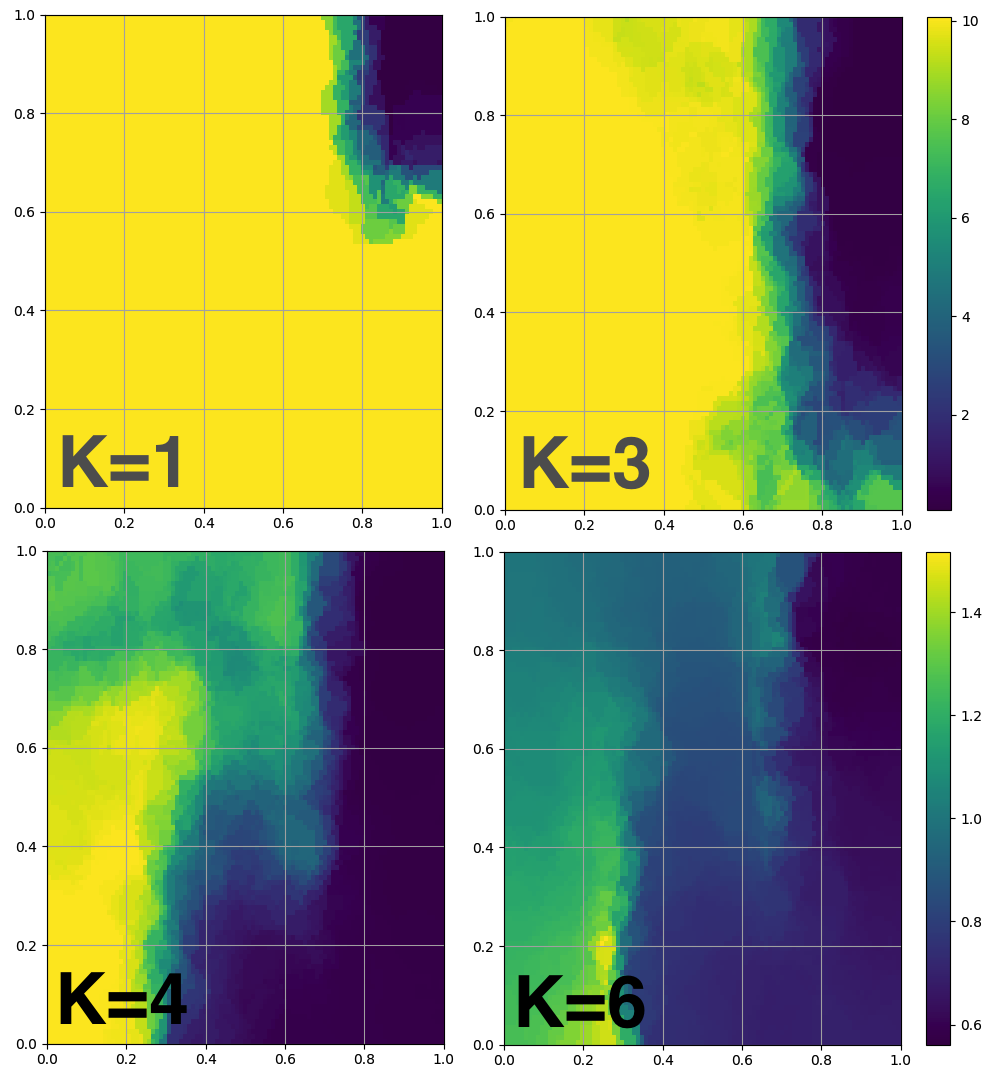}
    \caption{}
    \label{fig:rl_values}
    \end{subfigure}
    \hfill
    \vspace{-12px}
\caption{
Batch RL experiment. (a) A robot needs to navigate between the (hatched) obstacles. Blue - SGT prediction, green - trajectory tracking sub-goals using an inverse model. 
(b) Approximate values $\hat{V}_k(s,g=[0.9,0.9])$ for several values of $k$. Note how the reachable region to the goal (non-yellow) grows with $k$.
}
\label{fig:scenarios}
\vspace{-5mm}
\end{figure}
\begin{figure*}
\centering
    \hfill
    \begin{subfigure}[b]{0.2\linewidth}
    \includegraphics[width=\textwidth]{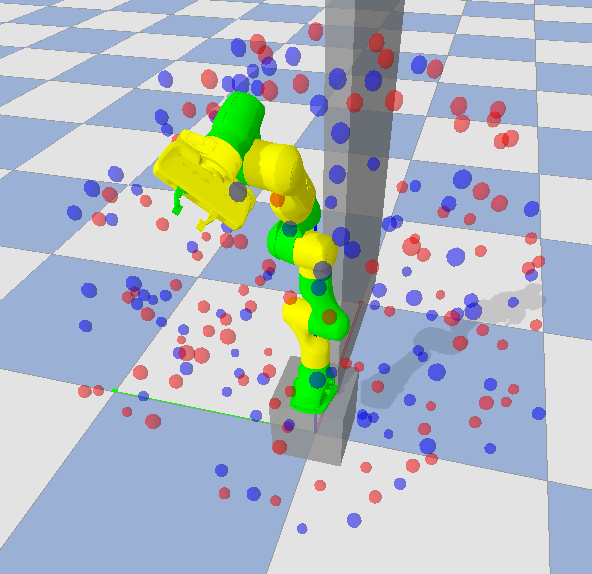}
    \caption{}
    \label{fig:robot-start-goals}
    \end{subfigure}
    \hfill
    \begin{subfigure}[b]{0.73\linewidth}
    \includegraphics[width=\textwidth]{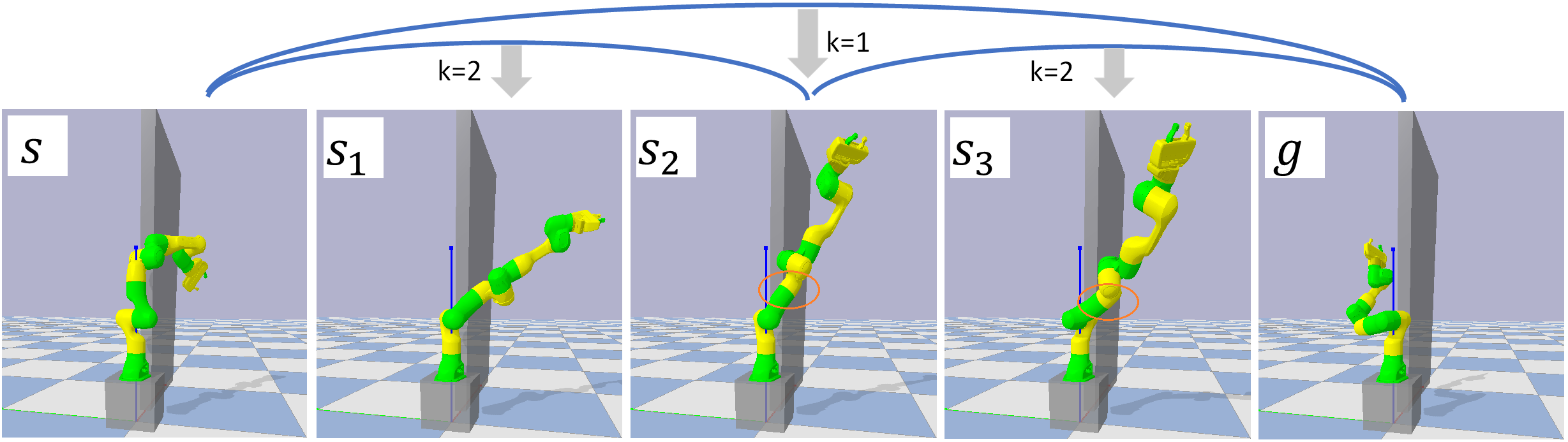}
    \caption{}
    \label{fig:robot-track}
    \end{subfigure}
    \hfill
    \vspace{-13px}
    \caption{
    Motion planning with \textit{SGT-PG}.
    (a) Illustration of start (red) and goal (blue) positions of the end effector in \textit{wall} domain.
    (b) SGT Trajectory. 
    Note a non-trivial rotation between $s_2$ to $s_3$ (marked with orange circles) allowing linear motion from $s_3$ to $g$. 
    In the supplementary material we also show a 7 sub-goals path.
}
\label{fig:nmp}
\end{figure*}

\subsection{Fitted SGTDP Experiments}\label{sec:batch-rl-experiments}
We start by evaluating the Fitted SGTDP algorithm. We consider a 2D particle moving in an environment with obstacles, as shown in Figure~\ref{fig:rl}. The particle can move a distance of $0.025$ in one of the eight directions, and suffers a constant cost of $0.025$ in free space, and a large cost of $10$ on collisions. Its task is reaching from any starting point to within a $0.15$ distance of any goal point without collision. 
This simple domain is a continuous-state optimal control Problem~\eqref{eq:opt_control}, and for distant start and goal points, as shown in Figure~\ref{fig:rl}, it requires relatively long-horizoned planning, making it suitable for studying batch RL algorithms.

To generate data, we sampled states and actions uniformly and independently, resulting in $125$K $(s,u,c,s')$ tuples. As for function approximation, we opted for simplicity, and used K-nearest neighbors (KNN) for all our experiments, with K$_{neighbors}=5$. To solve the minimization over states in Fitted SGTDP, we discretized the state space and searched over a $50\times 50$ grid of points.

A natural baseline in this setting is FQI~\cite{ernst2005tree, riedmiller2005neural}. We  verified that for a fixed goal, FQI obtains near perfect results with our data. Then, to make it goal-conditioned, we used a universal Q-function~\cite{schaul2015universal}, requiring only a minor change in the algorithm (see supplementary for pseudo-code). 

To evaluate the different methods, we randomly chose 200 start and goal points, and measured the distance from the goal the policies reach, and whether they collide with obstacles along the way. For FQI, we used the greedy policy with respect to the learned Q function. The approximate SGTDP method, however, does not automatically provide a policy, but only a state trajectory to follow. Thus, we experimented with two methods for extracting a policy from the learned sub-goal tree. The first is training an inverse model $f_{IM}(s,s')$ -- a mapping from $s,s'$ to $u$, using our data, and the same KNN function approximation. To reach a sub-goal $g$ from state $s$, we simply run $f_{IM}(s,g)$ until we are close enough to $g$ (we set the threshold to $0.15$). An alternative method is first using FQI to learn a goal-based policy, as described above, and then running this policy on the sub-goals. The idea here is that the sub-goals learned by approximate SGTDP can help FQI overcome the long-horizon planning required in this task.
Note that \emph{all methods use exactly the same data, and the same function approximation} (KNN), making for a fair comparison.

In Table \ref{tab:RL-experiments_results} we report our results. 
FQI only succeed in reaching the very closest goals, resulting in a high average distance to goal.
Fitted SGTDP (approx. SGTDP+IM), on the other hand, computed meaningful sub-goals for almost all test cases, resulting in a low average distance to goal when tracked by the inverse model. Figure~\ref{fig:scenarios} shows an example sub-goal tree and a corresponding tracked trajectory. The FQI policy did learn not to hit obstacles, resulting in the lowest collision rate. This is expected, as colliding leads to an immediate high cost, while the inverse model is not trained to take cost into account. Interestingly, combining the FQI policy with the sub-goals improves both long-horizon planning and short horizoned collision avoidance. In Figure \ref{fig:rl_values} we plot the approximate value function $\hat{V}_k$ for different $k$ and a specific goal at the top right corner. Note how the reachable parts of the state space to the goal expand with $k$.
\begin{table}[]
\centering
\begin{tabular}{c|c|c}
                    & \begin{tabular}[c]{@{}c@{}}Avg. Distance\\ to Goal\end{tabular} & \begin{tabular}[c]{@{}c@{}}Avg. Collision\\ Rate\end{tabular} \\ \hline
approx. SGTDP + IM  & 0.13                                                            & 0.25                                                          \\
approx. SGTDP + FQI & 0.29                                                            & 0.06                                                          \\
FQI                 & 0.58                                                            & 0.02                                                         
\end{tabular}
\caption{Results for controllers of batch-RL experiments.}
\label{tab:RL-experiments_results}
\vspace{-20px}
\end{table}

\subsection{Neural Motion Planning}\label{sec:sgt-pg-experiments}



An interesting application domain for our approach is neural motion planning~\citep[NMP,][]{qureshi2018motion} -- learning to predict collision-free trajectories for a robot among obstacles. Here, we study NMP for the 7DoF Franka Panda robotic arm. Due to lack of space, full technical details of this section appear in supplementary Section \ref{sec:nmp-technical-details}. 


We follow an RL approach to NMP~\cite{jurgenson2019harnessing}, using a cost function that incentivizes short, collision-free trajectories that reach the goal. 
The state space is the robot's 7 joint angles, and actions correspond to predicting the next state in the plan. Given the next predicted state, the robot moves by running a fixed PID controller in simulation, tracking a linear joint motion from current to next state, and a cost is incurred based on the resulting motion.



We formulate NMP as approximate APSP as follows. A model predicts $T-1$ sub-goals, resulting in $T$ motion segments. 
Those segments are executed and evaluated \textit{independently}, i.e. when evaluating segment $(s, s')$ the robot first resets to $s$, and the resulting cost is based on its travel to $s'$.

Our experiments include two scenarios: \textit{self-collision} and \textit{wall}.
In \textit{self-collision}, there are no obstacles and the challenge is to generate a minimal distance path while avoiding self-collisions between the robot links. 
The more challenging \textit{wall} workspace contains a wall that partitions the space in front of the robot (see Figure~\ref{fig:nmp}). 
In \textit{wall}, the shortest path is often nonmyopic, requiring to first move \textit{away} from the goal in order to pass the obstacle.

    
We compare \textit{SGT-PG} with a sequential baseline, \textit{Sequential sub-goals (SeqSG)}, which prescribes the sub-goals predictions \textit{sequentially}.
For appropriate comparison, both models use a PPO objective~\citep{schulman2017proximal}, and a fixed architecture neural network to model the policy. All other hyper-parameters were specifically tuned for each model. 

\begin{table}[]
\centering
\begin{tabular}{c|l|c|c}
Model                   & \begin{tabular}[c]{@{}l@{}}\# Sub-\\ goals\end{tabular} & \textit{\begin{tabular}[c]{@{}c@{}}self-\\ collision\end{tabular}} & \textit{wall}  \\ \hline
\multirow{3}{*}{SGT-PG} & 1                                                      & \textbf{1.$\pm$0.}                                                      & 0.896$\pm$0.016          \\
                        & 3                                                      & \textbf{1.$\pm$0.}                                                      & \textbf{0.973$\pm$0.007} \\
                        & 7                                                      & 0.996$\pm$0.007                                                         & \textbf{0.973$\pm$0.007} \\ \hline
\multirow{3}{*}{SeqSG}  & 1                                                      & \textbf{1.$\pm$0.}                                                      & 0.676$\pm$0.034          \\
                        & 3                                                      & 0.983$\pm$0.007                                                         & 0.593$\pm$0.047          \\
                        & 7                                                      & 0.88$\pm$0.03                                                           & 0.487$\pm$0.037         
\end{tabular}
\caption{Success rates for the NMP scenarios.}
\label{tab:sgt-pg-results}
\vspace{-20px}
\end{table}

Table~\ref{tab:sgt-pg-results} compares \textit{SGT-PG} and \textit{SeqSG}, on predictions of 1, 3, and 7 sub-goals.
We evaluate success rate (reaching goal without collision) on 100 random start-goal pairs that were held out during training (see Figure \ref{fig:nmp}a).
Each experiment was repeated 3 times and the mean and range are reported.
Note that only 1 sub-goal is required to solve \textit{self-collision}, and both models obtain perfect scores. For \textit{wall}, on the other hand, more sub-goals are required, and here SGT significantly outperforms \textit{SeqSG}, which was not able to accurately predict several sub-goals.


\section{Conclusion}
We presented a framework for multi-goal RL that is derived from a novel first principle -- the SGT dynamic programming equation. For deterministic domains, we showed that SGTs are less prone to drift due to approximation errors, reducing error accumulation from $O(N^2)$ to $O(N \log N)$. We further developed value-based and policy gradient RL algorithms for SGTs, and demonstrated that, in line with their theoretical advantages, SGTs demonstrate improved performance in practice.

Our work opens exciting directions for future research, including: (1) can our approach be extended to stochastic environments? (2) how to explore effectively based on SGTs? (3) can SGTs be extended to image-based tasks? Finally, we believe that our ideas will be important for robotics and autonomous driving applications, and other domains where goal-based predictions are important.

\section*{Acknowledgements}
This work is partly funded by the Israel Science Foundation (ISF-759/19) and the Open Philanthropy Project Fund, an advised fund of Silicon Valley Community Foundation. 
Finally, the authors would like to thank Orr Krupnik for helpful comments.

\bibliography{example_paper}
\bibliographystyle{icml2020}

\appendix
\onecolumn
\section{Proofs}
\textbf{Proof of Theorem \ref{thm:SGTDP}}

\begin{proof}
First, by definition, the shortest path from $s$ to itself is $0$. In the following, therefore, we assume that $s\neq s'$. 

We show by induction that each $V_k(s,s')$ in Algorithm \eqref{eq:DP_trajsplit} is the cost of the shortest path from $s$ to $s'$ in $2^k$ steps or less.

Let $\tau_k(s,s')$ denote a shortest path from $s$ to $s'$ in $2^k$ steps or less, and let $V_k(s,s')$ denote its corresponding cost. 
Our induction hypothesis is that $V_{k-1}(s, s')$ is the cost of the shortest path from $s$ to $s'$ in $2^{k-1}$ steps or less. We will show that $V_k(s, s') = \min_{s_m} \left\{ V_{k-1}(s, s_m) + V_{k-1}(s_m, s')\right\}$.

Assume by contradiction that there was some $s^*$ such that $V_{k-1}(s, s^*) + V_{k-1}(s^*, s') < V_k(s,s')$. Then, the concatenated trajectory $[\tau_{k-1}(s,s^*), \tau_{k-1}(s^*,s')]$ would have $2^k$ steps or less,
contradicting the fact that $\tau_k(s,s')$ is a shortest path from $s$ to $s'$ in $2^k$ steps or less. So we have that $V_k(s, s') \leq \left\{ V_{k-1}(s, s_m) + V_{k-1}(s_m, s')\right\} \quad \forall s_m$. 
Since the graph is complete, $\tau_k(s,s')$ can be split into two trajectories of length $2^{k-1}$ steps or less. Let $s_m$ be a midpoint in such a split. 
Then we have that $V_k(s, s') = \left\{ V_{k-1}(s, s_m) + V_{k-1}(s_m, s')\right\}$. 
So equality can be obtained and thus we must have $V_k(s, s') = \min_{s_m} \left\{ V_{k-1}(s, s_m) + V_{k-1}(s_m, s')\right\}$.

To complete the induction argument, we need to show that $V_0(s, s') = c(s, s')$. This holds since for $k=0$, for each $s,s'$, the only possible trajectory between them of length $1$ is the edge $s,s'$. 

Finally, since there are no negative cycles in the graph, for any $s,s'$, the shortest path has at most $N$ steps. Thus, for $k= \log_2 (N)$, we have that  $V_k(s, s')$ is the cost of the shortest path in $N$ steps or less, which is the shortest path between $s$ and $s'$.
\end{proof}



\textbf{Proof of Proposition \ref{prop:err_prop}}
\begin{proof}
The SGTDP Operator T is non-linear and not a contraction, but it is monotonic:
\begin{equation}\label{eq:t_monotonic}
    \forall s,g: V_\alpha(s,g) \leq V_\beta(s,g) \rightarrow \forall s,g: TV_\alpha(s,g) \leq TV_\beta(s,g).
\end{equation}
To show (\ref{eq:t_monotonic}), let $s_m'=\arg\min_{s_m}{\{V_{\beta}(s,s_m) + V_{\beta}(s_m,g)\}}$. Then:
\begin{align*}
    TV_{\alpha}(s,g)&=\min_{s_m}{\{V_{\alpha}(s,s_m) + V_{\alpha}(s_m,g)\}}
    \le V_{\alpha}(s,s_m') + V_{\alpha}(s_m',g)\\
    &\le V_{\beta}(s,s_m') + V_{\beta}(s_m',g)
    =\min_{s_m}{\{V_{\beta}(s,s_m) + V_{\beta}(s_m,g)\}}
    =TV_{\beta}(s,g)
\end{align*}

Back to the proof. By denoting $e=(1,1,1,1,...,1)$ we can write (\ref{eq:error}) as: \begin{align*}
    V_0 - e\epsilon \leq \hat{V_0} \leq V_0 + e\epsilon
\end{align*}
Since T is monotonic, we can apply it on the inequalities:
\begin{align*}
    T(V_0 - e\epsilon) \leq T\hat{V_0} \leq T(V_0 + e\epsilon)
\end{align*}
Focusing on the left-hand side expression (the right-hand side is symmetric) we obtain:
\begin{align*}
    T(V_0 - e\epsilon)(s,s') = \min_{s_m} \left\{ (V_0 - e\epsilon)(s, s_m) + (V_0 - e\epsilon )(s_m, s')\right\} = \min_{s_m} \left\{ V_0(s, s_m) + V_0(s_m, s') - 2\epsilon \right\} = TV_0(s,s') - 2\epsilon
\end{align*}
leading to:
\begin{align*}
    TV_0 - 2e\epsilon \leq T\hat{V_0} \leq TV_0 + 2e\epsilon
\end{align*}
Using (\ref{eq:error}):
\begin{align*}
    TV_0 - (2+1)e\epsilon \leq T\hat{V_0} - e\epsilon \leq \hat{V_1} \leq T\hat{V_0} + e\epsilon \leq TV_0 + (2+1)e\epsilon
\end{align*}
Proceeding similarly we obtain for $ \hat{V_2} $:
\begin{align*}
    T^2V_0 - [2(2+1)+1]e\epsilon \leq \hat{V_2} \leq T^2V_0 + [2(2+1)+1]e\epsilon
\end{align*}
And for every $ k \geq 1 $:
\begin{align*}
    T^kV_0 - (2^{k+1}-1)e\epsilon \leq \hat{V_k} \leq T^kV_0 + (2^{k+1}-1)e\epsilon
\end{align*}
\begin{equation}\label{v_k_bound}
    \| \hat{V}_k - V_k \|_\infty = \| \hat{V}_k - T^{k}V_0 \|_\infty \leq (2^{k+1}-1)\epsilon
\end{equation}
For $ k = \log_2{N} $ we obtain:
\begin{align*}
    \| \hat{V}_{\log_2{N}} - V^* \|_\infty = \| \hat{V}_{\log_2{N}} - V_{\log_2{N}} \|_\infty \leq \epsilon(2N-1)
\end{align*}
\end{proof}

\textbf{Proof of Proposition \ref{prop:STDP_error}}

\begin{proof}
For every iteration $k$, the middle state of any greedy SGT path with length $2^k$ is $ \hat{s}_m = \argmin_{s_m} \left\{ \hat{V}_{k-1}(s, s_m) + \hat{V}_{k-1}(s_m, g)\right\}$. Thereby $\hat{s}_m $ fulfils the identity $ \hat{V}_{k-1}(s, \hat{s}_m) + \hat{V}_{k-1}(\hat{s}_m, g) = T\hat{V}_{k-1}(s, g)$, for every iteration $k$, where $T$ is the SGT operator. The next two relations then follow:
\begin{align}
    | \hat{V}_{k-1}(s,\hat{s}_m) + \hat{V}_{k-1}(\hat{s}_m,g) - V_k(s,g)| &= | T\hat{V}_{k-1}(s,g) - V_k(s,g) | \leq \nonumber \\
    & \leq | T\hat{V}_{k-1}(s,g) - \hat{V}_k(s,g) | + | \hat{V}_k(s,g) - V_k(s,g) |  \leq \\
    & \leq \epsilon + (2^{k+1}-1) \epsilon  \nonumber= 2^{k+1}\epsilon \nonumber
\end{align}
\begin{align}
    | V_{k-1}(s,\hat{s}_m) + V_{k-1}(\hat{s}_m,g) - V_k(s,g) | \leq & | \hat{V}_{k-1}(s,\hat{s}_m) +  \hat{V}_{k-1}(\hat{s}_m,g) - V_k(s,g) | + \nonumber \\ 
    & + | V_{k-1}(s,\hat{s}_m) - \hat{V}_{k-1}(s,\hat{s}_m) | + \\ 
    & + | V_{k-1}(\hat{s}_m,g) -  \hat{V}_{k-1}(\hat{s}_m,g) | \leq \nonumber \\
    \leq & | \hat{V}_{k-1}(s,\hat{s}_m) + \hat{V}_{k-1}(\hat{s}_m,g) - V_k(s,g) | + 2\cdot(2^{k}-1) \epsilon \nonumber
\end{align}
Combining the two inequalities yields:
\begin{equation}\label{greedy_sgt_proof_sm_inequality}
        | V_{k-1}(s,\hat{s}_m) + V_{k-1}(\hat{s}_m,g) - V_k(s,g) | \leq (2^{k+1} + 2\cdot(2^{k}-1))\epsilon \leq 2^{k+2} \epsilon
\end{equation}
Explicitly writing down relation (\ref{greedy_sgt_proof_sm_inequality}) for all the different sub-paths we obtain:
\begin{alignat*}{4}
    &| V_{\log_2(N/2)}(s_0,s_{N/2}) &&+ V_{\log_2(N/2)}(s_{N/2}, s_N) &&- V_{\log_2(N)}(s_0,s_N) &&| \leq 4N \epsilon \\ \\
    &| V_{\log_2(N/4)}(s_0,s_{N/4}) &&+ V_{\log_2(N/4)}(s_{N/4}, s_{N/2}) &&- V_{\log_2(N/2)}(s_0,s_{N/2}) &&| \leq 2N \epsilon \\
    &| V_{\log_2(N/4)}(s_{N/2},s_{3N/4}) &&+ V_{\log_2(N/4)}(s_{3N/4}, s_N) &&- V_{\log_2(N/2)}(s_{N/2},s_N) &&| \leq 2N \epsilon\\ \\
    &| V_{\log_2(N/8)}(s_0,s_{N/8}) &&+ V_{\log_2(N/8)}(s_{N/8}, s_{N/4}) &&- V_{\log_2(N/4)}(s_0,s_{N/4}) &&| \leq N \epsilon\\
    &| V_{\log_2(N/8)}(s_{N/4},s_{3N/8}) &&+ V_{\log_2(N/8)}(s_{3N/8}, s_{N/2}) &&- V_{\log_2(N/4)}(s_{N/4},s_{N/2}) &&| \leq N \epsilon\\
    &| V_{\log_2(N/8)}(s_{N/2},s_{5N/8}) &&+ V_{\log_2(N/8)}(s_{5N/8}, s_{3N/4}) &&- V_{\log_2(N/4)}(s_{N/2},s_{3N/4}) &&| \leq N \epsilon\\
    &| V_{\log_2(N/8)}(s_{3N/4},s_{7N/8}) &&+ V_{\log_2(N/8)}(s_{7N/8}, s_N) &&- V_{\log_2(N/4)}(s_{3N/4},s_N) &&| \leq N \epsilon \\ \intertext{\hfil \vdots \hfil}
    &| V_0(s_0,s_1) &&+ V_0(s_1, s_2) &&- V_1(s_0,s_2)  &&| \leq 8 \epsilon \\
    &| V_0(s_2,s_3) &&+ V_0(s_3, s_4) &&- V_1(s_2,s_4) &&| \leq 8 \epsilon \\
    \intertext{\hfil \vdots \hfil}
    &| V_0(s_{N-2},s_{N-1}) &&+ V_0(s_{N-1}, s_N) &&- V_1(s_{N-2},s_N) &&| \leq 8 \epsilon
\end{alignat*}
Summing all the inequalities, along with the triangle inequality, lead to the following:
\begin{align*}
| \sum_{i=0}^{N-1} V_0(s_i,s_{i+1}) - V_{\log_2(N)}(s_0,s_N) | \leq  4N \log_2{(N)} \epsilon
\end{align*}
The identities $c(s_i,s_{i+1}) = V_0(s_i, s_{i+1})$ and $ V^*(s_0,s_N) =  V_{\log_2{N}}(s_0,s_N)$ complete the proof.
\end{proof}

Proposition \ref{prop:STDP_error} provides a bound on the approximate shortest path using SGTDP. In contrast, we show that a similar bound for the sequential Bellman approach does not hold. For a start and goal pair $s,g$, and an approximate Bellman value function $\hat{V}^B$, let $s_0^B,\dots,s_N^B$ denote the greedy shortest path according to the Bellman operator, i.e., $s_0 = s, s_N = g$ and for $1 \leq k < N$: $s_{k+1} = \argmin_{s_m} \left\{ c(s_k,s_m) + \hat{V}^B(s_m, g)\right\}$.
Note that the greedy Bellman update is not guaranteed to generate a trajectory that reaches the goal, therefore we add $g$ as the last state in the trajectory, following our convention that the graph is fully connected. 
When evaluating the greedy trajectory error in the sequential approach, we deal with two different implementation cases: Using one approximated value function or $N$ approximated value functions. The next proposition shows that when using one value function the greedy trajectory can be arbitrarily bad, regardless of $\epsilon$. Propositions (\ref{prop:Bellman_multiple_v_error}) and (\ref{prop:Bellman_multiple_v_error_example}) show that the error of the greedy trajectory, using $N$ value functions, have a tight $\mathcal{O}(N^2)$ bound. 

\begin{proposition}\label{prop:Bellman_error}
For any $\epsilon,\beta$, there exists a graph and an approximate value function $\hat{V}^B$  satisfying $\| \hat{V}^B - V^* \|_\infty \leq \epsilon N$, such that $\sum_{i=0}^{N-1} c(s_i,s_{i+1}) \geq V^*(s,g) + \beta$.
\end{proposition}

\begin{proof}
We show a graph example (Figure \ref{fig:graph_bellman_one_v_not_reaching}), where a $\epsilon$-approximate Bellman value function $\hat{V}^{B}$ might form the path $s_0, s_0, \dots, s_0, g$ with a total cost of infinity. (Reminder: The graph is fully connected, all the non-drawn edges have infinity cost), while the optimal path $s_0, s_1, \dots, s_{N-1}, g$ has a total cost of $N\epsilon$.
\newline
As $\hat{V}^{B}$ is independent of the current time-step $k$, and the optimal values for $s$ and $s_1$ are $N \epsilon$ and $(N-1) \epsilon$ respectively, due to approximation error $\hat{V}^{B}$ (for every time-step) might suggest that the value of $s_0$ is lower than the value of $s_1$: $\hat{V}^{B}(s_0,g) -N\epsilon \leq \hat{V}^{B}(s_1,g) + N\epsilon$.
The resulting policy will choose to stay in $s_0$ for the first $N-1$ steps. Since the maximum path length is $N$ and the path must end at the goal, the last step will be directly from $s_0$ to $g$, resulting in a cost of infinity.

\end{proof}
\begin{figure*}
\centering
\includegraphics[width=1.0\textwidth]{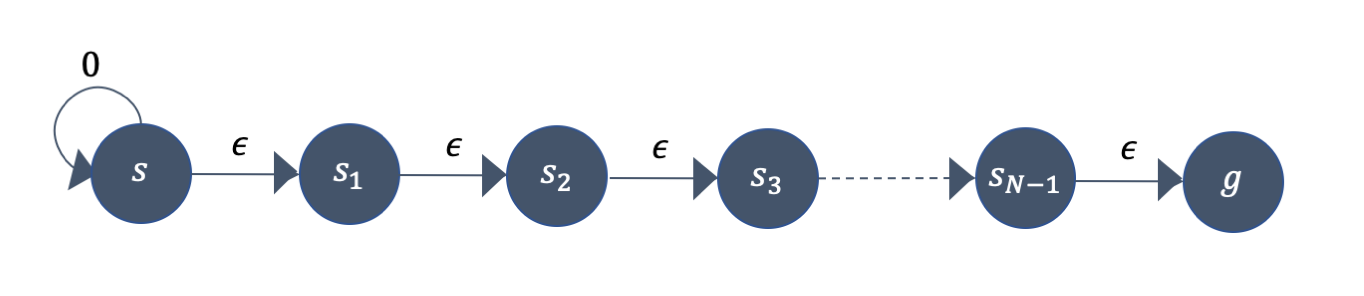}
\caption{Example graph with a start and goal states $s,g$ where the optimal path has a total cost of $N \epsilon$ while $\epsilon$-approximated Bellman might form a path with total cost of infinity}
\label{fig:graph_bellman_one_v_not_reaching}
\vspace{-5mm}
\end{figure*}

\begin{figure*}
\centering
\includegraphics[width=1.0\textwidth]{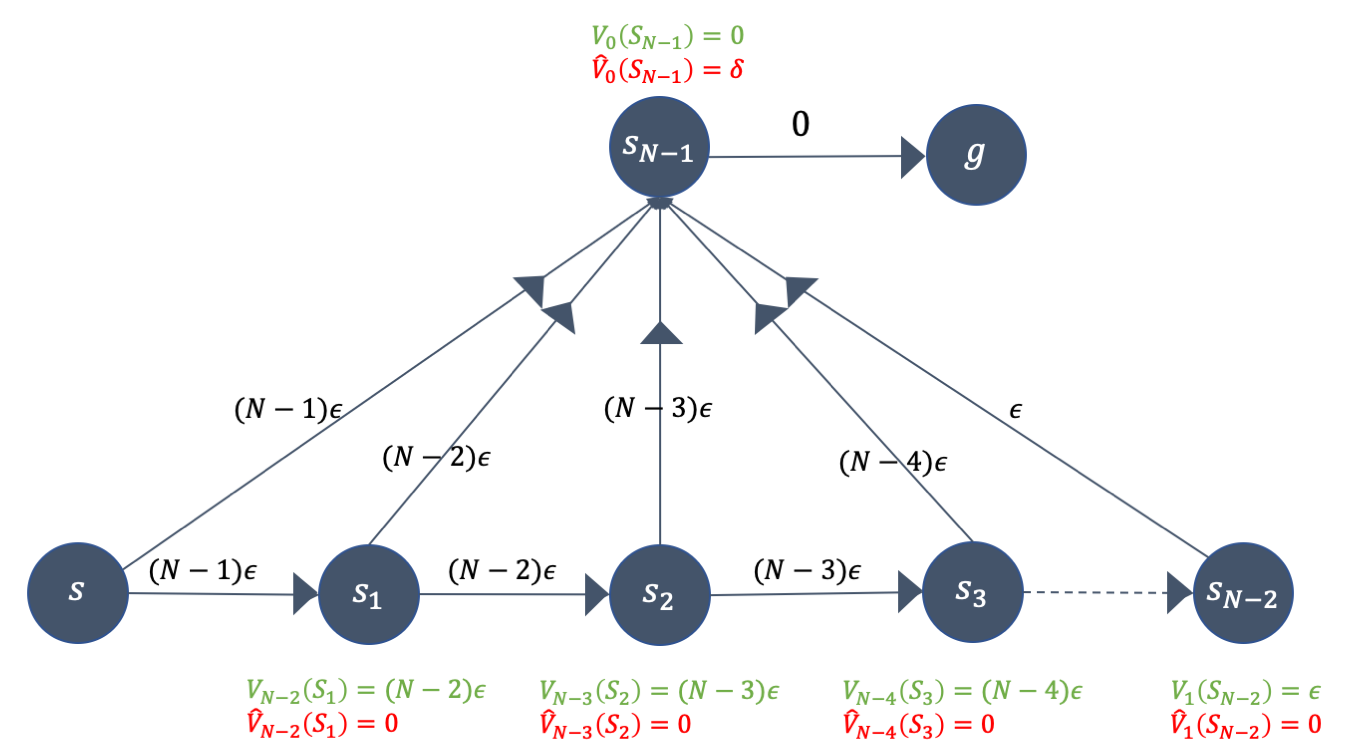}
\caption{Example graph with a start and goal states $s,g$. The edges costs are marked in black labels, the Bellman value function at iteration k is marked in green and the $\epsilon$-approximated Bellman value is marked in red. The optimal path has a total cost of $V^*(s,g)=(N-1)\epsilon$ while the $\epsilon$-approximated Bellman might form the path going through all the states with a total cost of ${N^2-N}\epsilon/2 = V^*(s,g) + \mathcal{O}(N^2)$.}
\label{fig:graph_bellman_multiple_v}
\vspace{-5mm}
\end{figure*}

\begin{proposition}\label{prop:Bellman_multiple_v_error}
For a finite state graph, start and goal pair $s,g$, and sequence of $\epsilon$-approximate Bellman value functions $\hat{V}^{B}_0,\dots, \hat{V}^{B}_{N}$ satisfying $\| \hat{V}^{B}_{k+1} - T^B \hat{V}^{B}_{k} \|_\infty \leq \epsilon $ and $\| \hat{V}^{B}_0 - V_0 \|_\infty \leq \epsilon $, let $s_0,\dots,s_N$ denote the greedy Bellman path, that is:
$s_0 = s, s_N = g, s_{k+1} = \argmin_{s_{k+1}} c(s_k, s_{k+1}) + \hat{V}^{B}_{N-k-2}(s_{k+1}, g),$ etc.
We have that $\sum_{i=0}^{N-1} c(s_i,s_{i+1}) \leq V^*(s,g) + (N^2 - N) \epsilon = V^*(s,g) + \mathcal{O}(N^2) $.
\end{proposition}

\begin{proof}
Denote $V_k$ as the Bellman value function at iteration $k$, $\hat{V}_k$ as the $\epsilon$-approximated Bellman value function and $T_B$ as the Bellman operator. Using the properties  $ \| T_B \hat{V}_{k-1} - \hat{V_k} \|_\infty \leq \epsilon $ and $ \| V_k - \hat{V_k} \|_\infty \leq k\epsilon $ ~\citep[][Pg. 332]{bertsekas1996neuro}, the following relation holds:

for $ 0 \leq k \leq N-2$:
\begin{align*}
c(s_k, s_{k+1}) &= c(s_k, s_{k+1}) + \hat{V}_{N-k-2}(s_{k+1}, g) - \hat{V}_{N-k-2}(s_{k+1}, g) \\
&= (T_B\hat{V}_{N-k-2})(s_{k}, g) - \hat{V}_{N-k-2}(s_{k+1}, g) \\
&\leq (\hat{V}_{N-k-1}(s_k, g) + \epsilon) - \hat{V}_{N-k-2}(s_{k+1}, g) \\
&\leq V_{N-k-1}(s_k, g) + (N-k-1)\epsilon + \epsilon - V_{N-k-2}(s_{k+1}, g) + (N-k-2)\epsilon \\
&\leq V_{N-k-1}(s_k, g) - V_{N-k-2}(s_{k+1}, g) + 2(N-k-1)\epsilon\\
\end{align*}
for $k=N-1$:
\begin{align*}
c(s_k, s_{k+1}) &= c(s_{N-1}, g) = V_0(s_{N-1}, g)
\end{align*}

Summing all the path costs:
\begin{align*}
\sum_{k=0}^{N-1} c(s_k, s_{k+1}) &= V_0(s_{N-1}, g) + \sum_{k=0}^{N-2} c(s_k, s_{k+1}) \\
&\leq V_0(s_{N-1}, g) + \sum_{k=0}^{N-2} V_{N-k-1}(s_k, g) - V_{N-k-2}(s_{k+1}, g) + 2(N-k-1)\epsilon \\
&=  V_0(s_{N-1}, g) + V_{N-1}(s_0, g) - V_0(s_{N-1}, g) + (N^2-N)\epsilon = V^*(s_0, g) + (N^2-N)\epsilon
\end{align*}
\end{proof}

\begin{proposition}\label{prop:Bellman_multiple_v_error_example}
The error bound stated in Proposition(\ref{prop:Bellman_multiple_v_error}),  $\sum_{i=0}^{N-1} c(s_i,s_{i+1}) - V^*(s,g) \leq \mathcal{O}(N^2) $, is indeed a tight bound. 
\end{proposition}
\begin{proof}
We show a graph example (Figure \ref{fig:graph_bellman_multiple_v}),where a sequence of $\epsilon$-approximate Bellman value functions $\hat{V}_0,\dots, \hat{V}_{N-1}$ might form the path from the start to the goal state, going through all the available states, with a total cost of $(N^2-N)/2 = \mathcal{O}(N^2)$, while the optimal path $(s_0, s_{N-1}, s_N)$ has a total cost of $V^*(s,g) = (N-1)\epsilon$. \newline
The Bellman value of every state $s_k$, $0 \leq k \leq s_{N-3}$: $V_{N-k-1}(s_k, g) = (N-k-1)\epsilon$, where $N-k-1$ is the time horizon from state $s_k$ to the goal. Since the $\epsilon$-approximate Bellman value function at every state $s_k$ might have a maximum approximation error of $(N-k-1)\epsilon$, it might suggest that $\hat{V}_{N-k-1}(s_k, g) = 0$ for every $1 \leq k \leq N-2$.
The Bellman value for state $s_{N-1}$ is 0 and the $\epsilon$-approximate Bellman value might be $\delta$ for some $0 < \delta \leq \epsilon$.
The greedy Bellman policy has to determine at every state $s_k$, $0 \leq k \leq s_{N-3}$, whether to go to $s_{k+1}$ or to $s_{N-1}$. Since both possible actions have an immediate cost of $(N-k-1)\epsilon$, the decision will only be based on the evaluation of $s_{k+1}$ and $s_{N-1}$ with the $\epsilon$-approximate Bellman value function: $\hat{V}_{N-k-1}(s_k, g) = 0 < \delta = \hat{V}_{0}(s_{N-1}, g)$. Therefore, the greedy bellman policy will decide at every state $s_k$ to go to $s_{k+1}$, forming the path with a total cost of $\sum_{k=1}^{N-1} k\epsilon = {N^2-N}\epsilon/2$ \newline
\end{proof}

\newpage

\section{Policy Gradient Theorem for Sub-goal Trees}\label{sec:sgt-pg-thm-proof}
In this section we provide the mathematical framework and tool we used in the Policy Gradient Theorem~\ref{thm:sgt-pg} in Section~\ref{sec:sgt-pg} which allows us to formulate the SGT-PG algorithm (Section~\ref{sec:sgt-pg}).
The SGT the prediction process (Eq.\eqref{eq:recursive-trajectory-likelihood}) no-longer decomposes sequentially like a MDP, therefore, we provide a different mathematical construct:

\textbf{Finite-depth Markovian sub-goal tree (FD-MSGT)}
is a process predicting sub-goals (or intermediate states) of a trajectory in a dynamical system as described in~\eqref{eq:cont_sp}. 
The process evolves trajectories recursively (as we describe in detail below), inducing a tree-like decomposition for the probability of a trajectory as described Eq.~\eqref{eq:recursive-trajectory-likelihood}.
In this work we consider trees with fixed levels $D$ (corresponding to finite-horizon MDP in sequential prediction RL with horizon $T=2^D$) \footnote{Extending MSGT to infinite recursion without a depth restrictions (similar to infinite-horizon MDP) is left for future work. }.
Formally, a FD-MSGT is comprised of $(S,\rho_0,c, D)$ where $S$ is the state space, $\rho_0$ is the initial start-goal pairs distribution, $c$ is a non-negative cost function obeying Eq.~\eqref{eq:cont_sp}, and $D$ is the depth of the tree \footnote{Note the crucial difference between MSGTs and MDPs. MSGTs operate on pairs of states from $S$, whereas MDPs, which are usually not goal conditioned, operate on single states.}.

\textbf{Recursive evolution of FD-MSGT:} 
Formally, an initial pair $(s_0=s, s_T=g)\sim\rho_0$ is sampled, defining the root of the tree.
Next, a policy $\pi(s_{\frac{T}{2}} | s_0, s_T)$ is used to predict the sub-goal $s_{\frac{T}{2}}$, creating two new tree nodes of depth $D$ corresponding to the segments $(s_0, s_{\frac{T}{2}})$ and $(s_{\frac{T}{2}}, s_T)$.
Recursively each segment is again partitioned using $\pi$ resulting in four tree nodes of level $D-1$ corresponding to segments $(s_0, s_{\frac{T}{4}})$, $(s_{\frac{T}{4}}, s_{\frac{T}{2}})$, $(s_{\frac{T}{2}}, s_{\frac{3T}{4}})$ and $(s_{\frac{3T}{4}}, s_T)$.
The process continues recursively until the depth of the tree is $1$.
This results in Eq.\eqref{eq:recursive-trajectory-likelihood}, namely,
\begin{align*}
\Pr_{\pi}[\tau | s, g]= \Pr_{\pi}[s_0, \dots, s_T | s, g] = \Pr_{\pi}[s_0, \dots, s_{\frac{T}{2}} | s, s_m] 
    \Pr_{\pi}[s_{\frac{T}{2}}, \dots, s_T | s_m, g] \pi(s_m | s, g).
\end{align*}
FD-MSGT results in $\sum_{d=0}^{D-1}{2^d}=2^D-1$ sub-goals, and $2^D+1$ overall states (including $s$ and $g$), setting $T=2^D$.
Finally, $c$ defines the cost of the prediction $c(\tau)=c_{0:T}$, and is evaluated on the leaves of the FD-MSGT tree according to Eq.~\eqref{eq:cont_sp}.
Figure~\ref{fig:robot-track} illustrates a FD-MSGT of $D=2$.
The objective of FD-MSGT is to find a policy $\pi:S^2\rightarrow S$ minimizing the expected cost of a trajectory
$J^{\pi} = \mathbb{E}_{\tau\sim\rho(\pi)}\left[c_{\tau} \right]$.

\textbf{Non-recursive formulation for trajectory likelihood:} We next derive a non-recursive formulation of Eq.~\eqref{eq:recursive-trajectory-likelihood}, using the notations defined in Theorem~\ref{thm:sgt-pg}, namely, 
$s^{i,d}=s_{(i-1)\cdot 2^d}$, $s_m^{i,d}=s_{(2i-1)\cdot 2^{d-1}}$, $g^{i,d}=s_{i\cdot 2^d}$, and $C_{\tau}^{i,d}=c_{(i-1)\cdot 2^d: i\cdot 2^d}$ is the sum of costs from $s^{i,d}$ to $g^{i,d}$ of $\tau$.
We re-arrange the terms in~\eqref{eq:recursive-trajectory-likelihood} grouping by depth resulting in a non-recursive formula,
\begin{align}\label{eq:explicit-trajectory-likelihood}
    \Pr_{\pi}[\tau |s,g]=\Pr_{\pi}[s_0, \dots s_T|s,g] =
    \prod_{d=1}^{D}{ \prod_{i=1}^{2^{D-d}}{\pi\left(s_m^{i,d} \bigg\vert s^{i,d}, g^{i,d}\right). } }
\end{align}{}

For example consider a FD-MSGT with $D=3$ ($T=8)$. 
Given $s_0=s$ and $s_8=g$, the sub-goals will be $s_1,\dots s_7$.
We enumerate the two products in Eq.~\eqref{eq:explicit-trajectory-likelihood} resulting in:
\begin{enumerate}
    
    \item \textbf{$D=3$: }Results in a single iteration of $i=1=2^{3-3}$.
    \begin{align*}
        s^{1,3} = s_{(1-1)\cdot 2^3}= s_0, \quad
        s_m^{1,3}=s_{(2\cdot 1 -1)\cdot 2^{3-1}
        }=s_4, \quad
        g^{1,3} = s_{1\cdot 2^3}= s_8
    \end{align*}{}
    
    \item \textbf{$D=2$: } In this case $i\in[1, 2^{3-2}=2]$.
    \begin{align*}
        s^{1,2} = s_{(1-1)\cdot 2^2}= s_0, \quad
        s_m^{1,2}=s_{(2\cdot 1 -1)\cdot 2^{2-1}
        }=s_2, \quad
        g^{1,2} = s_{1\cdot 2^2}= s_4 \\
        s^{2,2} = s_{(2-1)\cdot 2^2}= s_4, \quad
        s_m^{2,2}=s_{(2\cdot 2 -1)\cdot 2^{2-1}
        }=s_6, \quad
        g^{2,2} = s_{2\cdot 2^2}= s_8
    \end{align*}{}

    \item \textbf{$D=1$: } Finally, predicts the leaves resulting in $i\in [1, 2^{3-1}=4]$.
    \begin{align*}
        s^{1,1} = s_{(1-1)\cdot 2^1}= s_0, \quad
        s_m^{1,1}=s_{(2\cdot 1 -1)\cdot 2^{1-1}
        }=s_1, \quad
        g^{1,1} = s_{1\cdot 2^1}= s_2 \\
        s^{2,1} = s_{(2-1)\cdot 2^1}= s_2, \quad
        s_m^{2,1}=s_{(2\cdot 2 -1)\cdot 2^{1-1}
        }=s_3, \quad
        g^{2,1} = s_{2\cdot 2^1}= s_4 \\
        s^{3,1} = s_{(3-1)\cdot 2^1}= s_4, \quad
        s_m^{3,1}=s_{(2\cdot 3 -1)\cdot 2^{1-1}
        }=s_5, \quad
        g^{3,1} = s_{3\cdot 2^1}= s_6 \\
        s^{4,1} = s_{(4-1)\cdot 2^1}= s_6, \quad
        s_m^{4,1}=s_{(2\cdot 4 -1)\cdot 2^{1-1}
        }=s_7, \quad
        g^{4,1} = s_{4\cdot 2^1}= s_8
    \end{align*}{}
\end{enumerate}{}
Allowing us to assert the equivalence of the explicit and recursive formulations in this case.

\subsection{Proof of Theorem~\ref{thm:sgt-pg}}

Let $\pi_{\theta}$ be a policy with parameters $\theta$, we next prove a policy gradient theorem (Theorem~\ref{thm:sgt-pg-helper1}) for computing $\nabla_{\theta}J^{\pi_{\theta}}$, using the following proposition:
\begin{proposition}\label{prop:prop_sgt_pg_helper1}
Let $\pi_{\theta}$ be a policy with parameters $\theta$ of a FD-MSGT with depth $D$, then 
\begin{align}\label{eq:prop_sgt_pg_helper1}
    \nabla_{\theta}\log{\Pr_{\rho(\pi_{\theta})}[\tau]} = 
    \sum_{d=1}^{D}{ \sum_{i=1}^{2^{D-d}}{\nabla_{\theta}\log {\pi_{\theta}\left(s_m^{i,d} \bigg\vert s^{i,d}, g^{i,d}\right) } } }
\end{align}{}
\end{proposition}

\begin{proof}
First, we express $\Pr_{\rho(\pi_{\theta})}[\tau]$ using Eq. \ref{eq:explicit-trajectory-likelihood} and $\rho_0$ obtaining, 
\begin{align*}
    \Pr_{\rho(\pi_{\theta})}[\tau] = \rho(s_0,s_T)\cdot 
    \prod_{d=1}^{D}{ \prod_{i=1}^{2^{D-d}}{\pi_{\theta}\left(s_m^{i,d} \bigg\vert s^{i,d}, g^{i,d}\right). } }
\end{align*}{}
Next, by taking the log we have,
\begin{align*}
    \log {\Pr_{\rho(\pi_{\theta})}[\tau]} = \log{\rho(s_0,s_T)} + 
    \sum_{d=1}^{D}{ \sum_{i=1}^{2^{D-d}}{\log {\pi_{\theta}\left(s_m^{i,d} \bigg\vert s^{i,d}, g^{i,d}\right) } } }.
\end{align*}{}
and taking the gradient w.r.t $\theta$ yields Eq.~\eqref{eq:prop_sgt_pg_helper1}
\end{proof}

Proposition~\ref{prop:prop_sgt_pg_helper1} shows that the gradient of a trajectory w.r.t. $\theta$ does not depend on the initial distribution. 
This allows us to derive a policy gradient theorem:

\begin{theorem}\label{thm:sgt-pg-helper1}
Let $\pi_{\theta}$ be a stochastic SGT policy,  $\rho(\pi_{\theta})$ be a trajectory distribution defined above, and $T=2^D$.
Then 
\begin{align}\label{eq:pg-thm-gradient-short-helper1}
    \nabla_{\theta}J(\theta)=\mathbb{E}_{\rho(\pi_{\theta})}\left[ c_{\tau}\cdot \nabla_{\theta}\log{\Pr_{\rho(\pi_{\theta})}[\tau] }\right]  
    =\mathbb{E}_{\rho(\pi_{\theta})}\left[
    c_{\tau}\cdot\sum_{d=1}^{D}{ \sum_{i=1}^{2^{D-d}}{
    \nabla_{\theta}\log{\pi_{\theta}\left(s_m^{i,d} \bigg\vert s^{i,d}, g^{i,d}\right) } } }
    \right]\!. \nonumber
\end{align}
\end{theorem}

\begin{proof}
To obtain $\nabla_{\theta}J(\theta)$ we write Eq.~\eqref{eq:cost-function-J} as an explicit expectation and use $\nabla_x f(x) = f(x) \cdot \nabla_x \log{f(x)}$:
\begin{align*}
    \nabla_{\theta}J(\theta)=&\sum_{\tau}{c(\tau) \cdot \nabla_{\theta}\Pr_{\rho(\pi_{\theta})}[\tau]}
    =\sum_{\tau}{c(\tau) \cdot \Pr_{\rho(\pi_{\theta})}[\tau] \cdot \nabla_{\theta}\log{\Pr_{\rho(\pi_{\theta})}[\tau] }}
    =\mathbb{E}_{\rho(\pi_{\theta})}\left[c(\tau) \nabla_{\theta}\log{\Pr_{\rho(\pi_{\theta})}[\tau] }\right]
\end{align*}{}
This proves the first equality. For the second equality we substitute $\nabla_{\theta}\log{\Pr_{\rho(\pi_{\theta})}[\tau]}$ according to Eq.~\eqref{eq:prop_sgt_pg_helper1}.
\end{proof}

The policy gradient theorem allows estimating $\nabla J(\theta)$ from on policy data collected using $\pi_{\theta}$. 
We next show how to improve the estimate in Theorem~\ref{thm:sgt-pg-helper1} using control variates (baselines).
We start with the following baseline-reduction proposition.

\begin{proposition}\label{prop:sgt-pg-helper-baseline}
Let $\pi_{\theta}$ be a policy with parameters $\theta$ of a FD-MSGT with depth $D$, and let $b^{i,d}=b(s^{i,d}, g^{i,d})$ be any fixed function $b^{i,d}:S^2\rightarrow R$, then 
\begin{equation}\label{eq:sgt-pg-baseline-prop}
    \mathbb{E}_{\rho(\pi_{\theta})}\left[
    \sum_{d=1}^{D}{ \sum_{i=1}^{2^{D-d}}{
    b^{i,d}\cdot\nabla_{\theta}\log{\pi_{\theta}\left(s_m^{i,d} \bigg\vert s^{i,d}, g^{i,d}\right) } } }
    \right] = 0.
\end{equation}{}
\end{proposition}

\begin{proof}
First we establish a useful property. 
Let $p_\theta(z)$ be some parametrized distribution. 
Differentiating $\sum\nolimits_{z } {{p_\theta }(z)}  = 1$  yields
\begin{equation}\label{eq:helper}
\sum\nolimits_{z} {(\nabla \log {p_\theta }({z})} ){p_\theta }({z}) = {\mathbb E^\theta }(\nabla \log {p_\theta }({z})) = 0.
\end{equation}
Then we expand the left-hand side of Eq.~\eqref{eq:sgt-pg-baseline-prop}, and use Eq.~\eqref{eq:helper} in the last row:
\begin{align}
    \mathbb{E}_{\rho(\pi_{\theta})}\left[
    \sum_{d=1}^{D}{ \sum_{i=1}^{2^{D-d}}{
    b^{i,d}\cdot\nabla_{\theta}\log{\pi_{\theta}\left(s_m^{i,d} \bigg\vert s^{i,d}, g^{i,d}\right) } } }
    \right]= 
    &\sum_{d=1}^{D}{ \sum_{i=1}^{2^{D-d}}{\mathbb{E}_{\rho(\pi_{\theta})}\left[
    b^{i,d}\cdot\nabla_{\theta}\log{\pi_{\theta}\left(s_m^{i,d} \bigg\vert s^{i,d}, g^{i,d}\right) } 
    \right]} }
    = \nonumber \\
    &\sum_{d=1}^{D}{ \sum_{i=1}^{2^{D-d}}{
    \mathbb{E}_{\rho(\pi_{\theta})}\left[
    \mathbb{E}_{\rho(\pi_{\theta})}\left[
    b^{i,d}\cdot\nabla_{\theta}\log{\pi_{\theta}\left(s_m^{i,d} \bigg\vert s^{i,d}, g^{i,d}\right) } 
    \big\vert s^{i,d}, g^{i,d} \right]
    \right]
    } }
    = \nonumber \\
    &\sum_{d=1}^{D}{ \sum_{i=1}^{2^{D-d}}{
    \mathbb{E}_{\rho(\pi_{\theta})}\left[ b^{i,d} \cdot 
    \mathbb{E}_{\rho(\pi_{\theta})}\left[
    \nabla_{\theta}\log{\pi_{\theta}\left(s_m^{i,d} \bigg\vert s^{i,d}, g^{i,d}\right) } 
     \big\vert s^{i,d}, g^{i,d} \right]
    \right]
    } }
    = \nonumber \\
    &\sum_{d=1}^{D}{ \sum_{i=1}^{2^{D-d}}{
    \mathbb{E}_{\rho(\pi_{\theta})}\left[ b^{i,d} \cdot 0
    \right]
    } }
    =0, \nonumber
\end{align}{}
which concludes the proof.

\end{proof}{}

We define $b:S^2\rightarrow R$ as a \textit{segment-dependent baseline} if $b$ is a fixed function that operates on a pair of states $s$ and $g$. 
The last proposition allows us to reduce segment-dependent  baselines, $b^{i,d}$, from the estimations of $\nabla J(\theta)$ without bias.
Finally, in the next proposition we show that instead of estimating $\nabla J(\theta)$ using $c_{\tau}$ we can instead use $C^{i,d}$ as follows:

\begin{proposition}\label{prop:variance-reduction-helper}
Let $\pi_{\theta}$ be a policy with parameters $\theta$ of a FD-MSGT with depth $D$, then:
\begin{align*}
    \nabla_{\theta}J(\theta)
    =
    \sum_{d=1}^{D}{ \sum_{i=1}^{2^{D-d}}{
    \mathbb{E}_{\rho(\pi_{\theta})}\left[
    c_{(i-1)\cdot 2^d:i\cdot 2^d}\cdot
    \nabla_{\theta}\log{\pi_{\theta}\left(s_m^{i,d} \bigg\vert s^{i,d}, g^{i,d}\right) }
    \right]
    } }
\end{align*}{}
\end{proposition}

\begin{proof}
We have that:
\begin{align}\label{eq:sgt-variance-reduction-helper}
    \nabla_{\theta}J(\theta)
    = &\mathbb{E}_{\rho(\pi_{\theta})}\left[
    c_{\tau}\cdot\sum_{d=1}^{D}{ \sum_{i=1}^{2^{D-d}}{
    \nabla_{\theta}\log{\pi_{\theta}\left(s_m^{i,d} \bigg\vert s^{i,d}, g^{i,d}\right) } } }
    \right]
    = 
    \sum_{d=1}^{D}{ \sum_{i=1}^{2^{D-d}}{
    \mathbb{E}_{\rho(\pi_{\theta})}\left[
    c_{\tau}\cdot
    \nabla_{\theta}\log{\pi_{\theta}\left(s_m^{i,d} \bigg\vert s^{i,d}, g^{i,d}\right) }
    \right]
    } }
    \nonumber \\
    =&
    \sum_{d=1}^{D}{ \sum_{i=1}^{2^{D-d}}{
    \mathbb{E}_{\rho(\pi_{\theta})}\left[
    \mathbb{E}_{\rho(\pi_{\theta})}\left[
    c_{\tau}\cdot
    \nabla_{\theta}\log{\pi_{\theta}\left(s_m^{i,d} \bigg\vert s^{i,d}, g^{i,d}\right) }
    \bigg\vert s_0 \dots s^{i,d}, g^{i,d} \dots s_T \right]
    \right]
    } }.
\end{align}{}

The first transition is the expectation of sums, and the second is the smoothing theorem.
We next expand the inner expectation in Eq.~\eqref{eq:sgt-variance-reduction-helper}, and partition the sum of cost $c_{\tau}$ into three sums: between indices $0$ to $(i-1)\cdot 2^d$, $(i-1)\cdot 2^d$ to $i\cdot 2^d$, and finally from $i\cdot 2^d$ to $T$.
\begin{align*}
    &\mathbb{E}_{\rho(\pi_{\theta})}\left[
    c_{\tau}\cdot
    \nabla_{\theta}\log{\pi_{\theta}\left(s_m^{i,d} \bigg\vert s^{i,d}, g^{i,d}\right) }
    \bigg\vert s_0 \dots s^{i,d}, g^{i,d} \dots s_T \right]
    \nonumber \\
    &=
    \mathbb{E}_{\rho(\pi_{\theta})}\left[
    \left(c_{0:(i-1)\cdot 2^d} + c_{i\cdot 2^d:T}\right)\cdot
    \nabla_{\theta}\log{\pi_{\theta}\left(s_m^{i,d} \bigg\vert s^{i,d}, g^{i,d}\right) }
    \bigg\vert s_0 \dots s^{i,d}, g^{i,d} \dots s_T \right]
    \nonumber\\
    &+
    \mathbb{E}_{\rho(\pi_{\theta})}\left[
    c_{(i-1)\cdot 2^d:i\cdot 2^d}\cdot
    \nabla_{\theta}\log{\pi_{\theta}\left(s_m^{i,d} \bigg\vert s^{i,d}, g^{i,d}\right) }
    \bigg\vert s_0 \dots s^{i,d}, g^{i,d} \dots s_T \right].
\end{align*}
Next, by the definition of FD-MSGT the costs between $s_0$ to $s_{(i-1)\cdot 2^d}$ and $s_{i\cdot 2^d}$ to $s_T$, are independent of the policy predictions between indices $(i-1)\cdot 2^d$ to $i\cdot 2^d$ -- making them constants.
As we showed in Proposition~\ref{prop:sgt-pg-helper-baseline}, we obtain that this expectation is 0. 
Thus the expression for $\nabla_{\theta}J(\theta)$ simplifies to:
\begin{align}
    \nabla_{\theta}J(\theta)
    =&
    \sum_{d=1}^{D}{ \sum_{i=1}^{2^{D-d}}{
    \mathbb{E}_{\rho(\pi_{\theta})}\left[
    \mathbb{E}_{\rho(\pi_{\theta})}\left[
    c_{\tau}\cdot
    \nabla_{\theta}\log{\pi_{\theta}\left(s_m^{i,d} \bigg\vert s^{i,d}, g^{i,d}\right) }
    \bigg\vert s_0 \dots s^{i,d}, g^{i,d} \dots s_T \right]
    \right]
    } } 
    \nonumber \\
    =&
    \sum_{d=1}^{D}{ \sum_{i=1}^{2^{D-d}}{
    \mathbb{E}_{\rho(\pi_{\theta})}\left[
    \mathbb{E}_{\rho(\pi_{\theta})}\left[
    c_{(i-1)\cdot 2^d:i\cdot 2^d}\cdot
    \nabla_{\theta}\log{\pi_{\theta}\left(s_m^{i,d} \bigg\vert s^{i,d}, g^{i,d}\right) }
    \bigg\vert s_0 \dots s^{i,d}, g^{i,d} \dots s_T \right]
    \right]
    } }
    \nonumber \\
    =&
    \sum_{d=1}^{D}{ \sum_{i=1}^{2^{D-d}}{
    \mathbb{E}_{\rho(\pi_{\theta})}\left[
    c_{(i-1)\cdot 2^d:i\cdot 2^d}\cdot
    \nabla_{\theta}\log{\pi_{\theta}\left(s_m^{i,d} \bigg\vert s^{i,d}, g^{i,d}\right) }
    \right]
    } }
\end{align}{}
\end{proof}{}

Finally, combining Propositions~\ref{prop:sgt-pg-helper-baseline} and~\ref{prop:variance-reduction-helper} in Theorem~\ref{thm:sgt-pg-helper1} provides us Theorem~\ref{thm:sgt-pg} in the main text.
\newpage

\section{Batch RL Baselines}
In Algorithm \ref{alg:fitted_Q} we present a goal-based versions of fitted-Q iteration (FQI)~\cite{ernst2005tree} using universal function approximation~\cite{schaul2015universal}, which we used as a baseline in our experiments.

\begin{algorithm}[htp]\caption{Fitted Q with Universal Function Approximation}
  \SetAlgoLined\DontPrintSemicolon
  \SetKwProg{myalg}{Algorithm}{}{}
  \label{alg:fitted_Q}
  \setcounter{AlgoLine}{0}
  \myalg{ }{
  \nl Input: dataset $D = \left\{ s, u, c, s'\right\}$, Goal reached threshold $\delta$ \\
  \nl Create transition data set $D_{trans} = \left\{ s, u, s'\right\}$ and targets $T_{trans} = \left\{ c\right\}$ with $s,s',c$ taken from $D$\\
  \nl Fit $\hat{Q}(s,u,s')$ to data in $D_{trans}$ \\
  \For{$k: 1...K$}{
  \nl Create random goal data set $D_{goal} = \left\{ s, u, g\right\}$ and targets 
  $$T_{goal} = \left\{ \left\{ c(s,u) + \min_{u'}\hat{Q}(s', u', g)_{||s'-g||>\delta}\right\}\right\}$$ with $s,u,s'$ taken from $D$ and $g$ randomly chosen from states in $D$\\
  \nl Fit $\hat{Q}(s,u,s')$ to data in $D_{goal}$ \\
  }
  }
\end{algorithm}

Next, in Algorithm \ref{alg:approximate_FW} we present an approximate dynamic programming version of Floyd-Warshall RL~\cite{kaelbling1993learning} that corresponds to the batch RL setting we investigate. This algorithm was not stable in our experiments, as the value function converged to zero for all states (when removing the self transition fitting in line 7 of the algorithm, the values converged to a constant value).

\begin{algorithm}[htp]\caption{Approximate Floyd Warshall}
  \SetAlgoLined\DontPrintSemicolon
  \SetKwProg{myalg}{Algorithm}{}{}
  \label{alg:approximate_FW}
  \setcounter{AlgoLine}{0}
  \myalg{ }{
  \nl Input: dataset $D = \left\{ s, u, c, s'\right\}$, Maximum path cost $C_{max}$ \\
  \nl Create transition data set $D_{trans} = \left\{ s, s'\right\}$ and targets $T_{trans} = \left\{ c\right\}$ with $s,s',c$ taken from $D$\\
  \nl Create random transition data set $D_{random} = \left\{ s, s_{rand}\right\}$ and targets $T_{random} = \left\{ C_{max}\right\}$ with $s,s_{rand}$ randomly chosen from states in $D$\\
  \nl Create self transition data set $D_{self} = \left\{ s, s\right\}$ and targets $T_{self} =\left\{ 0\right\}$ with $s$ taken from $D$\\
  \nl Fit $\hat{V}(s,s')$ to data in $D_{trans}, D_{random}, D_{self}$ \\
  \For{$k: 1...K$}{
  \nl Create random goal and mid-point data set $D_{goal} = \left\{ s, g\right\}$ and targets 
  $$T_{goal} = \left\{ \min\left\{ \hat{V}(s, g), \hat{V}(s, s_m) + \hat{V}(s_m, g)\right\}\right\}$$ with $s,s_m,g$ randomly chosen from states in $D$\\
  \nl Create self transition data set $D_{self} = \left\{ s, s\right\}$ and targets $T_{self} =\left\{ 0\right\}$ with $s$ taken from $D$\\
  \nl Fit $\hat{V}(s,s')$ to data in $D_{goal},D_{self}$ \\
  }
  }
\end{algorithm}
\newpage

\section{Learning SGT with Supervised Learning}
In this section we describe how to learn SGT policies for deterministic, stationary, and finite time dynamical systems from labeled data generated by an expert. 
This setting is commonly referred to as Imitation Learning (IL).
Specifically, we are given a dataset $D^{\pi^*}=\left\{ \tau_i\right\}_{i=1}^N$ of $N$ trajectory demonstrations, provided by an expert policy $\pi^*$.
Each trajectory demonstration $\tau_i$ from the dataset $D^{\pi^*}$, is a sequence of $T_i$ states, i.e. $\tau_i = s_0^i, s_1^i \dots s_{T_i}^i$.\footnote{We note that limiting our discussion to states only can be easily extended to include actions as well by concatenating states and actions. However, we refrain from that in this work in order to simplify the notations, and remain consistent with the main text. } 
In this work, we assume a goal-based setting, that is, we assume that the expert policy generates  trajectories that lead to a goal state which is the last state in each trajectory demonstration. Our goal is to learn a policy that, given a pair of current state and goal state $(s,g)$, predicts the trajectory that $\pi^*$ would have chosen to reach $g$ from $s$. 

\subsection{Imitation learning with SGT}
In this section, we describe the learning objectives for \textit{sequential} and \textit{sub-goal tree} prediction approaches under the IL settings. 
We focus on the Behavioral Cloning~(BC)\cite{pomerleau1989alvinn} approach to IL, where a parametric model for a policy $\hat{\pi}$ with parameters $\theta$ is learned by maximizing the log-likelihood of observed trajectories in $D^{\pi^*}$, i.e.,
\begin{align}\label{eq:bc-max-likelihood-general}
    \theta^* = \arg\max_\theta { \mathbb{E}_{\tau_i\sim D^{\pi^*}}\left[
        \log P_{\hat{\pi}}(\tau_i=s_0^i,s_1^i, \dots s_T^i|s_0^i,s_{T}^i; \theta)
    \right] }.
\end{align}
Denote the horizon $T$ as the maximal number of states in a trajectory. 
For ease of notation we assume $T$ to be the same for all trajectories\footnote{For trajectories with $T_i < T$ assume $s_i^{T_i}$ repeats $T-T_i$ times, alternatively, generate middle states from data until $T_i=T$}. 
Also, let $s_i=s_0^i$ and $g_i=s_{T_i}^i$ denote the start and goal states for $\tau_i$.
We ask -- how to best represent the distribution $P_{\hat{\pi}}(\tau|s,g;\theta)$?

The \textit{sequential} trajectory representation~\cite{pomerleau1989alvinn,ross2011reduction,zhang2018auto,qureshi2018motion}, a popular approach, decomposes $P_{\hat{\pi}}(s_0,s_1, \dots s_T|s,g;\theta)$ by sequentially predicting states in the trajectory conditioned on previous predictions.
Concretely, let $h_t=s_0,s_1,\dots,s_t$ denote the history of the trajectory at time index $t$, the decomposition assumed by the \textit{sequential} representation is $P_{\hat{\pi}}(s_0,s_1, \dots s_T|s,g;\theta) = P_{\hat{\pi}}(s_1 | h_0, g ; \theta) P_{\hat{\pi}}(s_2 | h_1, g;\theta)\dots P_{\hat{\pi}}(s_T | h_{T-1}, g ; \theta)$.
Using this decomposition, \eqref{eq:bc-max-likelihood-general} becomes:
\begin{align}\label{eq:bc-max-likelihood-sequential}
    \theta^* = \arg\max_\theta { \mathbb{E}_{\tau_i\sim D^{\pi^*}}\left[
        \sum_{t=1}^{T}{\log{P_{\hat{\pi}}(s_{t+1}^i|h_t^i,g^i; \theta)}}
    \right] }.
\end{align}
We can learn $P_{\hat{\pi}}$ using a batch stochastic gradient descent (SGD) method.
To generate a sample $(s_t, h_{t-1}, g)$ in a training batch, a trajectory $\tau_i$ is sampled from $D^{\pi^*}$, and an observed state $s_t^i$ is further sampled from $\tau_i$. 
Next, the history $h_{t-1}^i$ and goal $g_i$ are extracted from $\tau_i$.
After learning, sampling a trajectory between $s$ and $g$ is a straight-forward iterative process, 
where the first prediction is given by $s_1 \sim P_{\hat{\pi}}(s|h_0=s,g; \theta)$, and every subsequent prediction is given by $s_{t+1} \sim P_{\hat{\pi}}(s|h_t=(s,s_1,\dots s_t),g; \theta)$.
This iterative process stops once some stopping condition is met (such as a target prediction horizon is reached, or a prediction is 'close-enough' to $g$).

\textbf{The Sub-Goal Tree Representation:}
In SGT we instead predict the middle-state in a divide-and-conquer approach. Following Eq.~\eqref{eq:recursive-trajectory-likelihood} we can redefine the maximum-likelihood objective as
\begin{equation}\label{eq:bc-max-likelihood-sub-goal}
    \!\theta^* \!\!=\!\! \arg\max_\theta { \mathbb{E}_{\tau_i\sim D^{\pi^*}}\!\!\!\left[
        \log{P_{\hat{\pi}}(s_{T/2}^i | s^i, g^i ; \theta)}\!\!+\!\log{P_{\hat{\pi}}(s_{T/4}^i | s^i, s_{T/2}^i;\theta)}\!\!+\!\log{P_{\hat{\pi}}(s_{3T/4}^i | s_{T/2}^i,g^i;\theta)}\! \dots 
    \right] }.
\end{equation}

To organize our data for optimizing Eq.~\eqref{eq:bc-max-likelihood-sub-goal}, we first sample a trajectory $\tau_i$ from $D^{\pi^*}$ for each sample in the batch. From $\tau_i$ we sample two states $s_a^i$ and $s_b^i$ and obtain their midpoint $s_{\frac{a+b}{2}}^i$.
Pseudo-code for sub-goal tree prediction and learning are provided in Algorithm~\ref{alg:sub-goal-sl-predict} and Algorithm~\ref{alg:sub-goal-sl-train}.

\begin{algorithm}[htp]\caption{Sub-Goal Tree BC Trajectory Prediction }
  \SetAlgoLined\DontPrintSemicolon
  \SetKwProg{myalg}{Algorithm}{}{}
  \SetKwFunction{predictSGT}{PredictSGT}
  \label{alg:sub-goal-sl-predict}
  \myalg{ }{
  \setcounter{AlgoLine}{0}
  \nl Input: parameters $\theta$ of parametric distribution $P_{\hat{\pi}}$, start state $s$, goal state $g$, max depth $K$ \\
  \nl \Return [$s$] + \predictSGT{$\theta$, $s$, $g$, $K$} +[$g$]\\
  }{}
  \setcounter{AlgoLine}{0}
  \SetKwProg{myproc}{Procedure}{}{}
  \myproc{\predictSGT{$\theta$, $s_1$, $s_2$, $k$}}{
  \If{$k>0$}{
  \nl Predict midpoint $s_m \sim P_{\hat{\pi}}(s_m |s_1, s_2;\theta)$\\
  \nl \Return\predictSGT{$\theta$, $s_1$, $s_m$, $k-1$} + [$s_m$] + \predictSGT{$\theta$, $s_m$, $s_2$, $k-1$} \\
  }
  }
\end{algorithm} 

\begin{algorithm}[htp]\caption{Sub-Goal Tree BC SGD-Based Training}
  \SetAlgoLined\DontPrintSemicolon
  \SetKwProg{myalg}{Algorithm}{}{}
  \label{alg:sub-goal-sl-train}
  \myalg{ }{
  \setcounter{AlgoLine}{0}
  \nl Input: dataset $D=\left\{ \tau_i = s_0^i, s_1^i \dots s_{T_i}^i \right\}_{i=1}^N$, train steps $M$, batch size $B$\\
  \nl Initialize parameters $\theta$ for parametric distribution $P_{\hat{\pi}}(s_m |s, g;\theta)$\\
  \For{$i: 1...M$ }{
  \nl Sample batch of size $B$, each sample: $\tau_i \sim D$, $s_1, s_2 \sim \tau_i$\\
  \nl $s_m\leftarrow$ Get midpoint of $[s_1, s_2]$ according to $\tau_i$ for all items in batch\\
  \nl Update $\theta$ by minimizing the negative log-likelihood loss: \begin{align*}
       L = -\frac{1}{B}\cdot\sum_{b=1}^B{\nabla_{\theta}\log{P_{\hat{\pi}}(s_m^b|s_1^b, s_2^b; \theta)}}
  \end{align*}\\
  }
  \Return $\theta$ \\
  }
\end{algorithm}

\subsection{Imitation Learning Experiments}

\begin{figure*}
\centering
    \hfill
    \begin{subfigure}[b]{0.4\linewidth}
    \includegraphics[width=.97\textwidth]{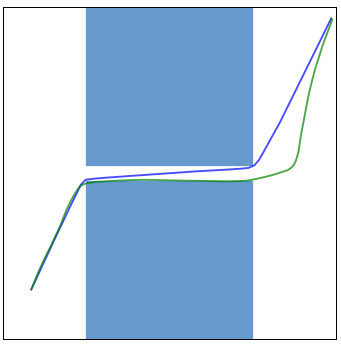}
    \caption{}
    \label{fig:easy}
  \end{subfigure}\hfill
  \begin{subfigure}[b]{0.4\linewidth}
    \includegraphics[width=.97\textwidth]{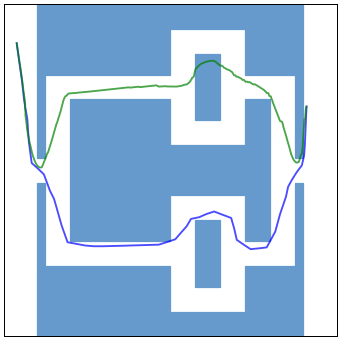}
    \caption{}
    \label{fig:hard}
  \end{subfigure}\hfill
\caption{IL Experiment domains and results. The \textit{simple} domain (a) and \textit{hard} domain (b). A point robot should move only on the white free-space from a room on one side to the room on the other side while avoiding the blue obstacles. SGT plan (blue) executed successfully, \textit{sequential} plan (green) collides with the obstacles. }
\label{fig:scenarios_il}
\end{figure*}

We compare the \textit{sequential} and \textit{Sub-Goal Tree} (SGT) approaches for BC.
Consider a point-robot motion-planning problem in a 2D world, with two obstacle scenarios termed \textit{simple} and \textit{hard}, as shown in Figure \ref{fig:scenarios_il}. In \textit{simple}, the distribution of possible motions from left to right is uni-modal, while in \textit{hard}, at least 4 modalities are possible.

For both scenarios we collected a set of 111K (100K train + 10K validation + 1K test)  expert trajectories from random start and goal states using a state-of-the-art motion planner (OMPL's\cite{sucan2012the-open-motion-planning-library} Lazy Bi-directional KPIECE \cite{csucan2009kinodynamic} with one level of discretization). 
To account for different trajectory modalities, we chose a Mixture Density Network (MDN)\cite{bishop1994mixture} as the parametric distribution of the predicted next state, both for the \textit{sequential} and the SGT representations. 
We train the MDN by maximizing likelihood using Adam~\cite{kingma2014adam}.
To ensure the same model capacity for both representations, we used the same network architecture, and both representations were trained and tested with the same data. Since the dynamical system is Markovian, the current and goal states are sufficient for predicting the next state in the plan, so we truncated the state history in the sequential model's input to contain only the current state.

For \textit{simple}, the MDNs had a uni-modal multivariate-Gaussian distribution, while for \textit{hard}, we experimented with 2 and 4 modal multivariate-Gaussian distributions, denoted as \textit{hard-2G}, and \textit{hard-4G}, respectively.
As we later show in our results the SGT representation captures the demonstrated distribution well even in the \textit{hard-2G} scenario, by modelling a bi-modal sub-goal distribution.

We evaluate the models using unseen start-goal pairs taken from the test set. 
To generate a trajectory, we predict sub-goals, and connect them using linear interpolation. 
We call a trajectory successful if it does not collide with an obstacle en-route to the goal. 
For a failed trajectory, we further measure the \textit{severity} of collision by the percentage of the trajectory being in collision.

Table \ref{tab:il-experiments_results} summarizes the results for both representations. The SGT representation is superior in all three evaluation criteria - motion planning success rate, trajectory prediction times (total time in seconds for the 1K trajectory predictions), and severity.
Upon closer look at the two \textit{hard} scenarios, the \textit{Sub Goal Tree} with a bi-modal MDN outperforms \textit{sequential} with 4-modal MDN, suggesting that the SGT trajectory decomposition better accounts for multi-modal trajectories.

\begin{table}[]
\centering
\begin{tabular}{l|ll|ll|ll}
\multirow{2}{*}{} & \multicolumn{2}{c|}{\textbf{Success Rate}} & \multicolumn{2}{c|}{\textbf{Prediction Times (Seconds)}} & \multicolumn{2}{c}{\textbf{Severity}} \\
                  & Sequential           & SGT            & Sequential               & SGT                 & Sequential           & SGT       \\ \hline
Simple            & 0.541                & \textbf{0.946}           & 487.179                  & \textbf{28.641}               & \textbf{0.0327}      & 0.0381              \\
Hard - 2G         & 0.013                & \textbf{0.266}           & 811.523                  & \textbf{52.22}                & 0.0803               & \textbf{0.0666}     \\
Hard - 4G         & 0.011                & \textbf{0.247}           & 1052.62                  & \textbf{53.539}               & 0.0779               & \textbf{0.0362}    
\end{tabular}
\caption{Results for IL experiments. 
}
\label{tab:il-experiments_results}
\vspace{-5mm}
\end{table}
\newpage

\section{Technical Details for NMP Experiments Section \ref{sec:sgt-pg-experiments}}\label{}
In this appendix we provide the full technical details of our policy gradient
experiments presented in Section~\ref{sec:sgt-pg-experiments}. 
We start by providing the explicit formulation of the NMP environment. 
We further provide a sequential-baseline not presented in the main text and its NMP environment and experimental results related to this baseline.
We also present another challenging scenario -- \textit{poles}, and the results of \textit{SGT-PG} and \textit{SeqSG} in it.
Finally, we conclude with modeling-related technical details \footnote{The code is attached to the submission (a GitHub project will be posted here upon acceptance)}.

\subsection{Neural Motion Planning Formulation}\label{sec:nmp-technical-details}

In Section~\ref{sec:sgt-pg-experiments}, we investigated the performance of \textit{SGT-PG} and \textit{SeqSG} executed in a NMP problem for the 7DoF Franka Panda robotic arm. 
That NMP formulation, denoted here as \textit{with-reset}, allows for independent evaluations of motion segments, a requirement for FD-MSGT. 
The state $s\in \mathbb{R}^9$ describes the robot's 7 joint angles and 2 gripper positions, normalized to $[-1,1]$.
The models, \textit{SGT-PG} and \textit{SeqSG}, are tasked with predicting the states (sub-goals) in the plan.

A segment $(s,s')$ is evaluated by running a PID controller tracking a linear joint motion from $s$ to $s'$. 
Let $s_{PID}$ denote the state after the PID controller was executed, and let $\alpha_{free}$ and $\alpha_{collision}$ be hyperparameters. We next define a cost function that encourages shorter paths, and discourages collisions;
if a collision occurred during the PID execution, or 5000 simulation steps were not enough to reach $s'$, a cost of $\alpha_{free}\cdot ||s-s_{PID}||_2 + \alpha_{collision}\cdot ||s'-s_{PID}||_2$ is returned.
Otherwise, the motion was successful, and the segment cost is $\alpha_{free}\cdot ||s-s'||_2$.
In our experiments we set $\alpha_{free}=1$ and $\alpha_{collision}=100$.
Figures~\ref{fig:self-collision-training-curve} and ~\ref{fig:wall-training-curve} show the training curves for the \textit{self-collision} and \textit{wall} scenarios respectively.

\subsection{The \textit{SeqAction} Baseline}
We also evaluated a sequential baseline we denoted as \textit{SeqAction}, which is more similar to classical RL approaches.
Instead of predicting the next state, \textit{SeqAction} predicts the next action to take (specifically the joint angles difference), and a position controller is executed for a single time-step instead of the PID controller of the previous methods. 

\textit{SeqAction} operates on a finite-horizon MDP (with horizon of $T=5000$ steps).
We follow the sequential NMP formulation of~\citet{jurgenson2019harnessing}; the agent gets a high cost on collisions, high reward when reaching close to the goal, or a small cost otherwise (to encourage shorter paths).
We denote this NMP variant \textit{no-reset} as segments are evaluated sequentially, without resetting the state between evaluations.
To make the problem easier, after executing an action in \textit{no-reset}, the velocity of all the links is manually set to 0.

Similarly to~\citet{jurgenson2019harnessing}, we found that sequential RL agents are hard to train for neural motion planning with varying start and goal positions without any prior knowledge or inductive bias. 
We trained \textit{SeqAction} in the \textit{self-collision} scenario using PPO~\citep{schulman2017proximal}, and although we achieved 100\% success rate when fixing both the start and the goal (to all zeros and all ones respectively), the model was only able to obtain less than 0.1 success rate when varying the goal state, and showed no signs of learning when varying both the start and the goal.


\subsection{The \textit{poles} Scenario}

\begin{figure*}
\centering
\includegraphics[width=1.0\textwidth]{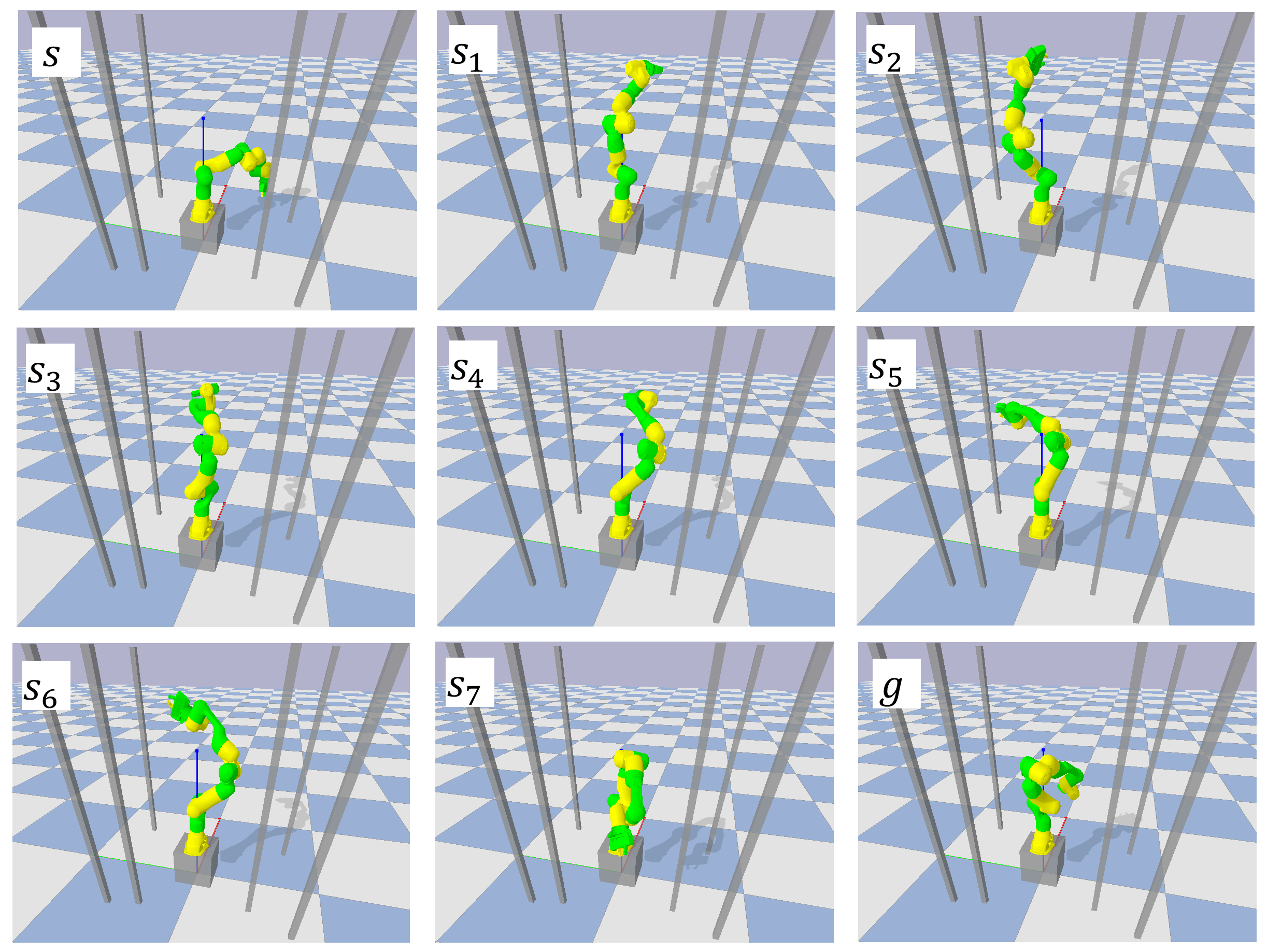}
\caption{Robot trajectory \textit{poles} scenario predicted with \textit{SGT-PG} with 7 sub-goals. Top left start state, bottom right goal state. Notice that the two final segments $s_6$ to $s_7$ and $s_7$ to $g$ are longer than the rest, but are made possible by the robot getting into position in state $s_6$. Note that \textit{SGT-PG} failed to find a trajectory for this $(s,g)$ pair with less sub-goals, indicating the hardness of the path.}
\label{fig:robot-trajectory-poles}
\vspace{-5mm}
\end{figure*}

\begin{figure*}
\centering
\includegraphics[width=0.8\textwidth]{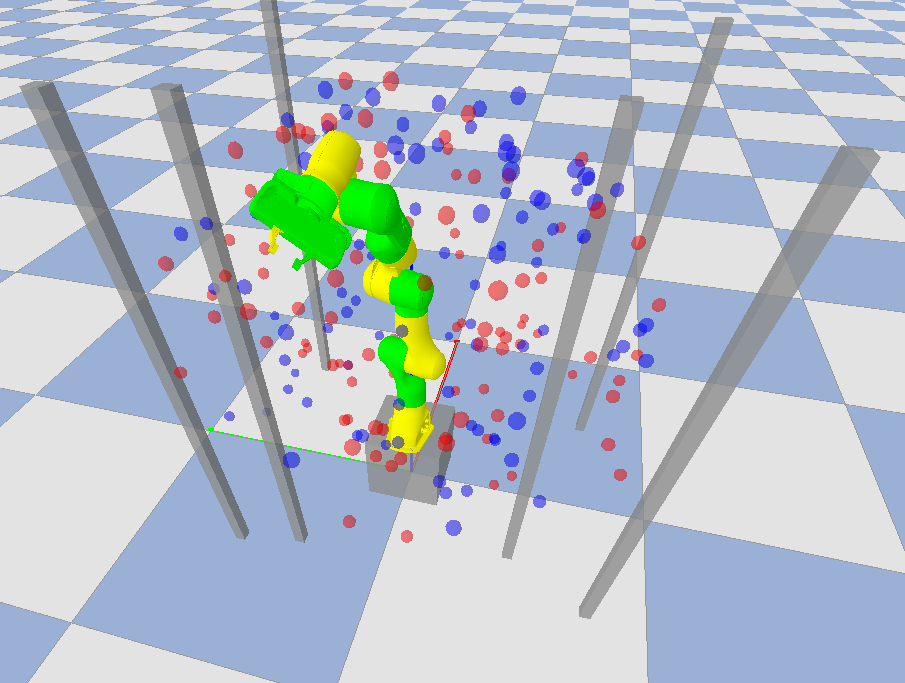}
\caption{Start (blue) and goal (red) states in the test set of the \textit{poles} scenario. Note that the agents are not exposed to these particular states during training.}
\label{fig:start-goal-pairs-poles}
\vspace{-5mm}
\end{figure*}

\begin{table}[]
\centering
\begin{tabular}{c|c|c}
model                   & \# Sub-goals & Success rate         \\ \hline
\multirow{3}{*}{SGT-PG} & 1            & 0.646$\pm$0.053          \\
                        & 3            & 0.663$\pm$0.153          \\
                        & 7            & \textbf{0.696$\pm$0.116} \\ \hline
\multirow{3}{*}{SeqSG}  & 1            & 0.616$\pm$0.013          \\
                        & 3            & 0.576$\pm$0.013          \\
                        & 7            & 0.49$\pm$0.02           
\end{tabular}
\caption{Success rate on the NMP poles scenario}
\label{tab:sgt-pg-results-poles}
\end{table}

We tested both \textit{SGT-PG} and \textit{SeqSG} on another scenario we call \textit{poles} where the goal is to navigate the robot between poles as shown in Figures~\ref{fig:robot-trajectory-poles} and~\ref{fig:start-goal-pairs-poles}. 
Similar to the previous scenarios, Table~\ref{tab:sgt-pg-results-poles} shows the success rates of both models for 1, 3, and 7 sub-goals. 
Again, we notice that \textit{SGT-PG} attains much better results and improves as more sub-goals are added, while the performance of \textit{SeqSG} deteriorates. 
Moreover, even for a single sub-goal, \textit{SGT-PG} obtains better success rates compared to \textit{SeqSG}.
The high range displayed by \textit{SGT-PG} is due to one random seed that obtained significantly lower scores.

\subsection{Neural Network Architecture and Training Procedure}\label{sec:nn-arch}

We next provide full technical details regarding the architecture and training procedure of \textit{SGT-PG} and \textit{SeqSG} described in Section~\ref{sec:sgt-pg-experiments}.

\textbf{Architecture and training of \textit{SGT-PG}:} The start and goal states, $s$ and $g$, are fed to a neural network with 3 hidden-layers of 20 neurons each and \textit{tanh} activations. 
The network learns a multivariate Gaussian-distribution over the predicted sub-goal with a diagonal covariance matrix. 
Since the covariance matrix is diagonal, the size of the output layer is $18=2\cdot 9$, 9 elements learn the bias, and 9 learn the standard deviation in each direction.  
To predict the bias, a value of $\frac{s+g}{2}$ is added to the prediction of the network. 
To obtain the final standard deviation, we add two additional elements to the standard deviation predicted by the neural network : (1) a small value (0.05), to prevent learning too narrow distributions, and (2) a learnable coefficient in each of the directions that depends on the distance between $s$ and $g$, which allows large segments to predict a wider spread of sub-goals. 
These two changes were incorporated in order to mitigate log-likelihood evaluation issues.

\textit{SGT-PG} trains on 30 $(s,g)$ pairs per training cycle, using the Adam optimizer~\citep{kingma2014adam} with a learning rate of 0.005 on the PPO objective~\citep{schulman2017proximal} augmented with a max-entropy loss term with coefficient 1, and a PPO trust region of $\epsilon=0.2$.
As mentioned in Section~\ref{sec:sgt-pg}, when training level $d$ for a single $(s,g)$ pair, we sample $M$ sub-goals (in our experiments $M=10$) from the multivariate Gaussian described above, take the mean predictions of policies of depth $d-1$ and lower, and take the mean of costs as a $(s,g)$-dependent baseline. 
The repetition allows us to obtain a stable baseline without incorporating another learnable module -- a value function.
We found this method to be more stable than training on 300 different start-goal pairs (no repetitions).
During test time we always take the mean prediction generated by the policies.

\textbf{Architecture of \textit{SeqSG}:} This model has two components, a policy and a value function. The policy architecture is identical to that of \textit{SGT-PG}. Note that only a single policy is maintained as the optimal policy is independent of the remaining horizon\footnote{Also, a policy-per-time-step approach will not scale to larger plans, as we cannot maintain $T=2^D$ different policies.}. 
We also found that learning the covaraince matrix causes stability issues for the algorithm, so the multivariate Gaussian only learns a bias term.
The value function of \textit{SeqSG} has 3 layers with 20 neurons each, and Elu activations. It predicts a single scalar which approximates $V(s,g)$, and is used as a state-dependent baseline to reduce the variance of the gradient estimation during training.

In order to compare apples-to-apples, \textit{SeqSG} also trains on 30 $(s,g)$ pairs with 10 repetitions each. 
We trained the policy with the PPO objective~\citep{schulman2017proximal} with trust region of $\epsilon=0.05$. We used the Adam optimizer~\citep{kingma2014adam}, with learning rate of $0.005$ but we clipped the gradient with L2 norm of 10 to stabilize training. 
The value function is trained to predict the value with mean-squared error loss. We also trained it with Adam with a learning rate of $0.005$ but clipped the gradient L2 norm to 100.

We note that finding stable parameters for \textit{SeqSG} was more challenging than for \textit{SGT-PG}, as apparent by the gradient clippings, the lower trust region value, and our failure to learn the mutivariate  Gaussian distribution covariance matrix.

\begin{figure*}
\centering
\includegraphics[width=0.8\textwidth]{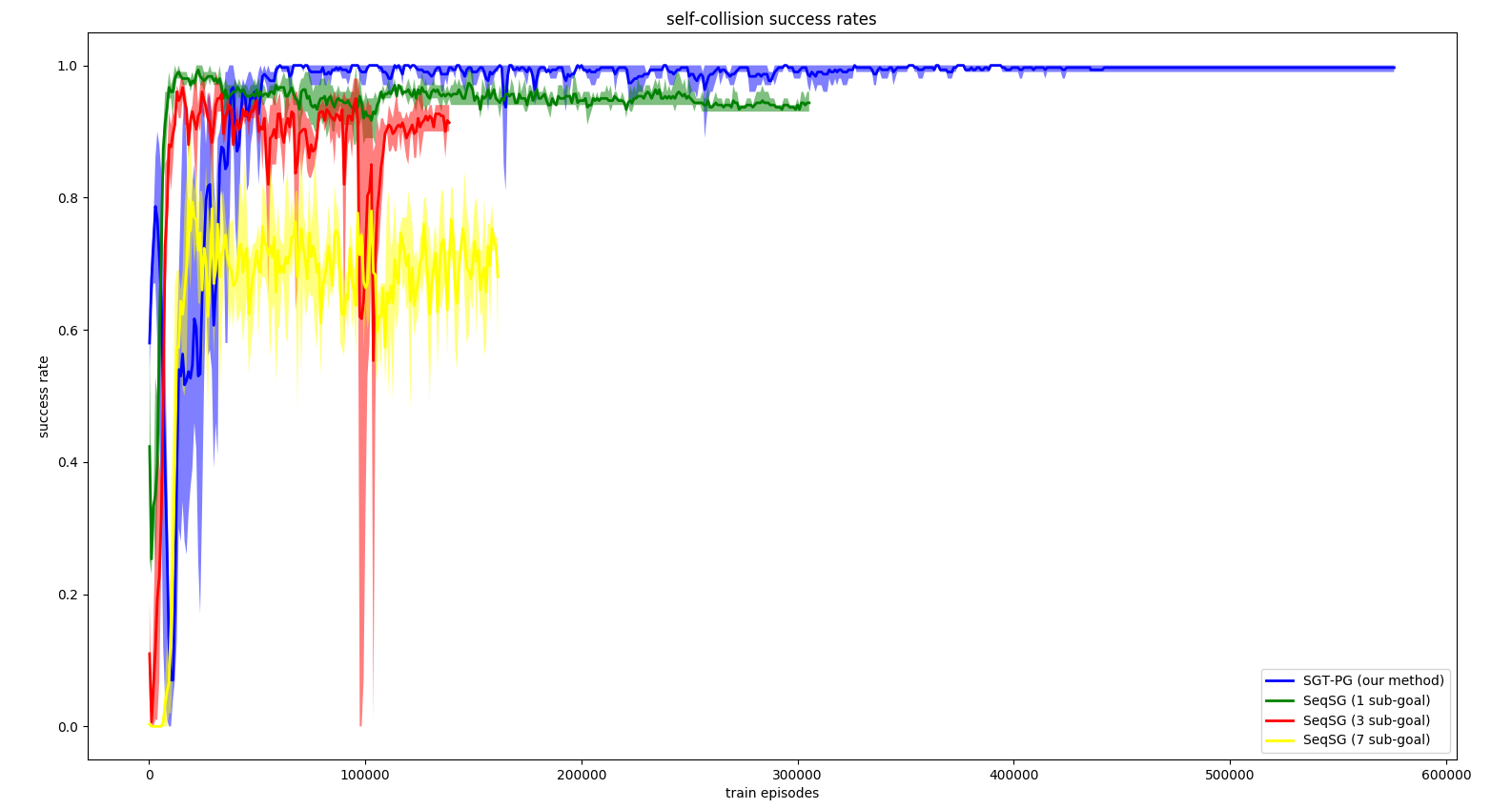}
\caption{Success rates over training episodes of the \textit{self-collision} scenario.}
\label{fig:self-collision-training-curve}
\vspace{-5mm}
\end{figure*}

\begin{figure*}
\centering
\includegraphics[width=0.8\textwidth]{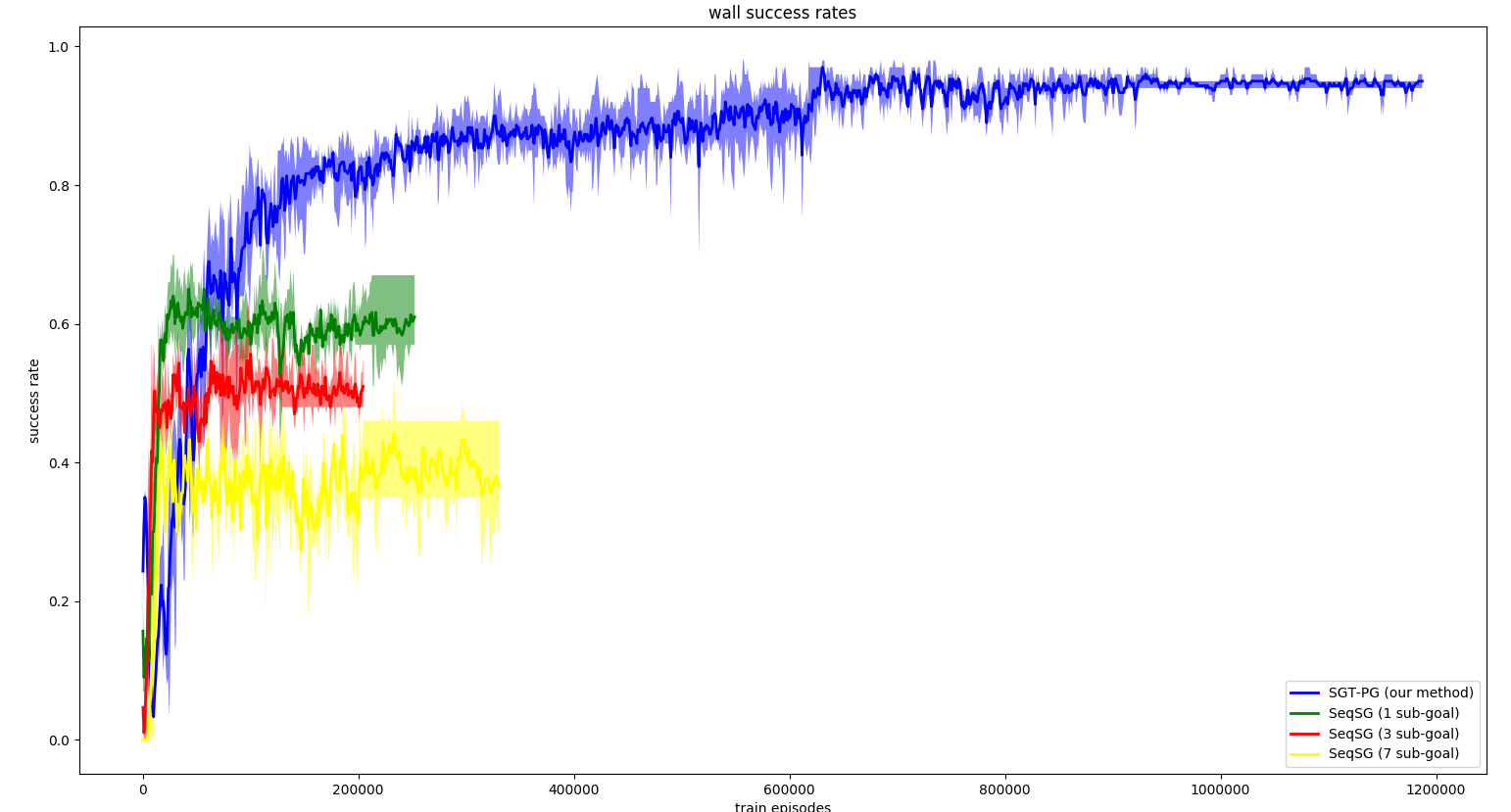}
\caption{Success rates over training episodes of the \textit{wall} scenario.}
\label{fig:wall-training-curve}
\vspace{-5mm}
\end{figure*}


\end{document}